\pdfoutput=1 % arXiv needs this
\documentclass{article}
\usepackage[utf8]{inputenc}
\usepackage[final]{neurips_2020}

\usepackage{natbib}
\usepackage{amsmath,amsthm,mathtools}
\newtheorem{theorem}{Theorem}

\newtheorem{proposition}[theorem]{Proposition}
\theoremstyle{definition}
\newtheorem{definition}[theorem]{Definition}
\newtheorem{example}[theorem]{Example}
\usepackage{amssymb,stmaryrd}
\DeclareMathOperator*{\argmax}{arg\,max}

\usepackage[hidelinks]{hyperref}
\usepackage{microtype}
\usepackage{algorithm}
\usepackage[noend]{algpseudocode}
\makeatletter
\let\algorithmicleftmargin\ALG@thistlm
\makeatother

\usepackage{calc}
\usepackage{tikz}
\usetikzlibrary{shapes,positioning,calc}
\tikzset{var/.style={draw,circle,fill=white,inner sep=1.5pt,minimum size=8pt}}
\tikzset{ext/.style={var,fill=black,text=white}}
\tikzset{fac/.style={draw,rectangle}}
\tikzset{subgraph/.style={draw,cloud}}
\tikzset{every picture/.style={baseline=-2.5pt}}
\tikzset{every label/.style={font={\footnotesize}}}
\tikzset{plate/.style={draw,rectangle,rounded corners}}

\usepackage{xcolor}
\usepackage{booktabs}
\usepackage{enumitem}
\newlist{compactenum}{enumerate}{1} 
\setlist[compactenum]{label=\arabic*.,ref=\arabic*,nosep,parsep=\parskip,leftmargin=2em}
\newlist{compactitem}{itemize}{3} 
\setlist[compactitem]{label=\textbullet,nosep,parsep=\parskip,leftmargin=2em}
\newcommand{\xlongrightarrow}[1]{\stackrel{#1}{\longrightarrow}}

\newcommand{\asst}{\xi}
\newcommand{\Asst}{\Xi}
\newcommand{\type}{\mathit{type}}
\newcommand{\att}{\mathit{att}}
\newcommand{\ext}{\mathit{ext}}
\newcommand{\lab}{\mathit{lab}}

\newcommand{\rv}[1]{\mathbf{#1}}
\newcommand{\nt}[1]{\mathsf{#1}} % use for actual nonterminal, not variables ranging over nonterminals
\newcommand{\pos}[1]{(#1)}

\newcommand{\labset}{L}
\newcommand{\deriv}{D}
\newcommand{\domain}{\Omega}

\let\conj\sqcap

\newcommand{\shrinkdisplay}[1]{\par\scalebox{0.8}{\begin{minipage}[b]{1.25\textwidth}\par#1\par\end{minipage}\par}}

\title{Factor Graph Grammars}
\author{David Chiang \\ University of Notre Dame \\ \texttt{dchiang@nd.edu} \And Darcey Riley \\ University of Notre Dame \\ \texttt{darcey.riley@nd.edu}}

\begin{document}

\maketitle

\begin{abstract}
We propose the use of hyperedge replacement graph grammars for factor graphs, or \emph{factor graph grammars} (FGGs) for short. FGGs generate sets of factor graphs and can describe a more general class of models than plate notation, dynamic graphical models, case--factor diagrams, and sum--product networks can. Moreover, inference can be done on FGGs without enumerating all the generated factor graphs. For finite variable domains (but possibly infinite sets of graphs), a generalization of variable elimination to FGGs allows exact and tractable inference in many situations. For finite sets of graphs (but possibly infinite variable domains), a FGG can be converted to a single factor graph amenable to standard inference techniques.
\end{abstract}

\section{Introduction}
\label{sec:intro}

Graphs have been used with great success as representations of probability models, both Bayesian and Markov networks \citep{koller+friedman:2009} as well as latent-variable neural networks \citep{schulman+:nips2015}. But in many applications, especially in speech and language processing, a fixed graph is not sufficient. The graph may have substructures that repeat a variable number of times: for example, a hidden Markov model (HMM) depends on the number of words in the string. Or, part of the graph may have several alternatives with different structures: for example, a probabilistic context-free grammar (PCFG) contains many trees for a given string.

Several formalisms have been proposed to fill this need. Plate notation \citep{buntine:jair1994}, plated factor graphs \citep{obermeyer+:icml2019}, and dynamic graphical models \citep{bilmes:2010} address the repeated-substructure problem, but only for sequence models like HMMs. Case--factor diagrams \citep{mcallester+:2008} and sum--product networks \citep{poon+domingos:2011} address the alternative-substructure problem, so they can describe PCFGs, but only for fixed-length inputs.

More general formalisms like probabilistic relational models \citep{getoor+:2007} and probabilistic programming languages \citep{vandemeent+:2018} address both problems successfully, but because of their generality, tractable exact inference in them is often not possible.

Here, we explore the use of \emph{hyperedge replacement graph grammars} (HRGs), a formalism for defining sets of graphs \citep{bauderon+courcelle:1987,habel+kreowski:1987,drewes+:1997}. We show that HRGs for factor graphs, or factor graph grammars (FGGs) for short, are expressive enough to solve both the repeated-substructure and alternative-substructure problems, and constrained enough allow exact and tractable inference in many situations. We make three main contributions:
\begin{compactitem}
\item We define FGGs and show how they generalize the constrained formalisms mentioned above~(\S\ref{sec:fgg}).
\item We define a \emph{conjunction} operation that enables one to modularize a FGG into two parts, one which defines the model and one which defines a query (\S\ref{sec:conjunction}).
\item We show how to perform inference on FGGs without enumerating the (possibly infinite) set of graphs they generate. For finite variable domains, we generalize variable elimination to FGGs (\S\ref{sec:inference_cfg}).
For some FGGs, this is exact and tractable; for others, it gives a sequence of successive approximations.

For infinite variable domains, we show that if a FGG generates a finite set, it can be converted to a single factor graph, to which standard graphical model inference methods can be applied (\S\ref{sec:inference_graph}). But if a FGG generates an infinite set, inference is undecidable (\S\ref{sec:inference_undecidable}).
\end{compactitem}

\section{Background}

In this section, we provide some background definitions for hypergraphs (\S\ref{sec:hypergraphs}), factor graphs (\S\ref{sec:factorgraphs}), and HRGs (\S\ref{sec:hrg}). Our definitions are mostly standard, but not entirely; readers already familiar with these concepts may skip these subsections and refer back to them as needed.

\subsection{Hypergraphs}
\label{sec:hypergraphs}

Assume throughout this paper the following ``global'' structures. Let $\labset^V$ be a finite set of \emph{node labels} and $\labset^E$ be a finite set of \emph{edge labels}, and assume there is a function $\type : \labset^E \rightarrow (\labset^V)^\ast$, which says for each edge label what the number and labels of the endpoint nodes must be. %Let the \emph{arity} of $\ell \in \labset^E$ be $|\type(\ell)|$.

\begin{definition}
A \emph{hypergraph} (or simply a \emph{graph}) is a tuple $(V, E, \att, \lab^V, \lab^E)$, where 
\begin{compactitem}
\item $V$ is a finite set of \emph{nodes}.
\item $E$ is a finite set of \emph{hyperedges} (or simply \emph{edges}).
\item $\att: E \rightarrow V^\ast$ maps each edge to zero or more \emph{endpoint} nodes, not necessarily distinct.
\item $\lab^V : V \rightarrow \labset^V$ assigns labels to nodes.
\item $\lab^E : E \rightarrow \labset^E$ assigns labels to edges. 
\item For all $e$, $|\att(e)| = |\type(\lab^E(e))|$, and
if $\att(e) = v_1 \cdots v_k$ and $\type(\lab^E(e)) = \ell_1 \cdots \ell_k$, then $\lab^V(v_i) = \ell_i$ for $i = 1, \ldots, k$.
\end{compactitem}
\end{definition}

Although the elements of $V$ and $E$ can be anything, we assume in our examples that they are natural numbers. If a node $v$ has label $\ell$, we draw it as a circle with $\ell_v$ inside it. We draw a hyperedge as a square with lines to its endpoints. In principle, we would need to indicate the ordering of the endpoints somehow, but we omit this to reduce clutter.

\subsection{Factor graphs}
\label{sec:factorgraphs}

%Let $\mathbb{K}$ be a commutative semiring (usually nonnegative reals or probabilities). We write $+$ (and $\sum$) for addition, and we use juxtaposition (and $\prod$) for multiplication.

\begin{definition} \label{def:factorgraph}
A \emph{factor graph} \citep{kschischang+:2001} is a hypergraph $(V, E, \att, \lab^V, \lab^E)$ together with mappings $\domain$ and $F$, where \begin{compactitem}
\item $\domain$ maps node labels to sets of possible values. For brevity, we write $\domain(v)$ for $\domain(\lab^V(v))$.
\item $F$ maps edge labels to functions. For brevity, we write $F(e)$ for $F(\lab^E(e))$. For every edge $e$ with $\att(e) = v_1 \cdots v_k$, $F(e)$ is of type $\domain(v_1) \times \cdots \times \domain(v_k) \rightarrow \mathbb{R}_{\geq 0}$.
\end{compactitem}
A node $v$ together with its domain $\domain(v)$ is called a \emph{variable}. An edge $e$ together with its function $F(e)$ is called a \emph{factor}.
\end{definition}
We draw a factor $e$ as a small square, but instead of writing its label, we write $F(e)$ next to it, as an expression in terms of its endpoints. As shorthand, we often write Boolean expressions, which are implicitly converted to real numbers ($\text{true} = 1$ and $\text{false} = 0$).

\begin{example} \label{eg:hmm3} Although HMMs are defined for sentences of arbitrary length, factor graphs force us to choose a fixed length; below is a HMM for sentences of length 3. (Here, $\rv{T}$ and $\rv{W}$ are node labels, $\domain(\rv{T})$ is the set of possible tags, and $\domain(\rv{W})$ is the set of possible words.)
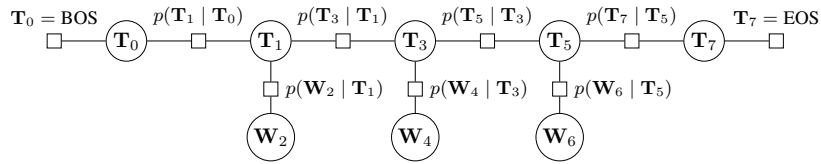
\shrinkdisplay{\begin{center}
\begin{tikzpicture}[x=1.2cm,y=0.8cm]
\node[var] (t0) at (0,0) {$\rv{T}_0$};
\node[var] (t1) at (2,0) {$\rv{T}_1$};
\node[var] (t2) at (4,0) {$\rv{T}_3$};
\node[var] (t3) at (6,0) {$\rv{T}_5$};
\node[var] (t4) at (8,0) {$\rv{T}_7$};
\node[fac,label=above:{$\rv{T}_0 = \text{BOS}$}] at (-1,0) {} edge (t0);
\node[fac,label=above:{$p(\rv{T}_1\mid \rv{T}_0)$}] at (1,0) {} edge (t0) edge (t1);
\node[fac,label=above:{$p(\rv{T}_3\mid \rv{T}_1)$}] at (3,0) {} edge (t1) edge (t2);
\node[fac,label=above:{$p(\rv{T}_5\mid \rv{T}_3)$}] at (5,0) {} edge (t2) edge (t3);
\node[fac,label=above:{$p(\rv{T}_7\mid \rv{T}_5)$}] at (7,0) {} edge (t3) edge (t4);
\node[fac,label=above:{$\rv{T}_7 = \text{EOS}$}] at (9,0) {} edge (t4);
\node[var] (w1) at (2,-2) {$\rv{W}_2$};
\node[var] (w2) at (4,-2) {$\rv{W}_4$};
\node[var] (w3) at (6,-2) {$\rv{W}_6$};
\node[fac,label=right:{$p(\rv{W}_2 \mid \rv{T}_1)$}] at (2,-1) {} edge (t1) edge (w1);
\node[fac,label=right:{$p(\rv{W}_4 \mid \rv{T}_3)$}] at (4,-1) {} edge (t2) edge (w2);
\node[fac,label=right:{$p(\rv{W}_6 \mid \rv{T}_5)$}] at (6,-1) {} edge (t3) edge (w3);
\end{tikzpicture}
\end{center}}
\end{example}

\begin{definition}
If $H$ is a factor graph, define an \emph{assignment} $\asst$ of $H$ to be a mapping from nodes to values: $\asst(v) \in \domain(v)$. We write $\Asst_H$ for the set of all assignments of $H$. 
The \emph{weight} of an assignment $\asst$ is given by
\begin{equation*}
w_H(\asst) = \prod_{\substack{\text{edges $e$}\\\mathclap{\text{with $\att(e) = v_1\cdots v_k$}}}} F(e)(\asst(v_1), \ldots, \asst(v_k)).
\end{equation*}
In a factor graph with no factors, every assignment has weight 1. A factor graph with no variables has exactly one assignment.
\end{definition}

Factor graphs are general enough to represent Bayesian networks and Markov networks.
They can also represent stochastic computation graphs (SCGs), introduced by \citet{schulman+:nips2015} for latent-variable neural networks. 

\subsection{Hyperedge Replacement Graph Grammars}
\label{sec:hrg}

Hyperedge replacement graph grammars (HRGs) were introduced by \citet{bauderon+courcelle:1987} and \citet{habel+kreowski:1987}, and surveyed by \cite{drewes+:1997}. They generate graphs by using a context-free rewriting mechanism that replaces nonterminal-labeled edges with graphs. In this section, we provide a brief definition of HRGs, with a minor extension for node labels.

\begin{definition}
A \emph{hypergraph fragment} is a tuple $(V, E, \att, \lab^V, \lab^E, \ext)$, where
\begin{compactitem}
\item $(V, E, \att, \lab^V, \lab^E)$ is a hypergraph,
\item $\ext \in V^\ast$ is a sequence of zero or more \emph{external nodes}.
\end{compactitem}
\end{definition}
In our figures, we draw external nodes as black nodes. In principle, we would need to indicate their ordering somehow, but we omit this to reduce clutter.

\begin{definition} \label{def:hrg}
A \emph{hyperedge replacement graph grammar} (HRG) is a tuple $(N, T, P, S)$, where
\begin{compactitem}
\item $N \subseteq \labset^E$ is a finite set of \emph{nonterminal symbols}. %(Although these are sometimes called ``variables,'' we reserve the term ``variables'' to refer to random variables.)
\item $T \subseteq \labset^E$ is a finite set of \emph{terminal symbols}, such that $N \cap T = \emptyset$.
\item $P$ is a finite set of \emph{rules} of the form $(X \rightarrow R)$, where 
\begin{compactitem}
\item $X \in N$.
\item $R$ is a hypergraph fragment with edge labels in $N \cup T$.
\item If $R$ has external nodes $x_1 \cdots x_k$, then $\type(X) = \lab^V(x_1) \cdots \lab^V(x_k)$.
\end{compactitem}
\item $S \in N$ is a distinguished \emph{start nonterminal symbol} with $\type(S) = \epsilon$.
\end{compactitem}
\end{definition}
Although a left-hand side $X$ is formally just a nonterminal symbol, we draw it as a hyperedge labeled~$X$ inside, with replicas of the external nodes as its endpoints. 
On right-hand sides, we draw an edge~$e$ with nonterminal label $X$ as a square with $X_e$ inside. If $R$ is the empty graph, we write $\emptyset$.

Intuitively, a HRG generates graphs by starting with a hyperedge labeled $S$ and repeatedly selecting an edge $e$ labeled $X$ and a rule $X \rightarrow R$ and replacing $e$ with $R$. (See Figure~\ref{fig:replacement}a for an example, where $H=R$.) Replacement stops when there are no more nonterminal-labeled edges.

As with a CFG, we can abstract away from the ordering of replacement steps using a \emph{derivation tree}, in which the nodes are labeled with HRG rules, and an edge from parent $\pi_1$ to child $\pi_2$ has a label indicating which edge in the right-hand side of $\pi_1$ is replaced with the right-hand side of $\pi_2$.

\begin{figure}
\begin{center}
\begin{tabular}{c@{\hspace*{4em}}c}
\begin{tikzpicture}
\node[var](v1) at (90:0.6) {};
\node[var](v2) at (210:0.6) {};
\node[var](v3) at (330:0.6) {};
\coordinate(u) at (80:1) edge (v1);
\coordinate(u) at (100:1) edge (v1);
\coordinate(u) at (200:1) edge (v2);
\coordinate(u) at (220:1) edge (v2);
\coordinate(u) at (320:1) edge (v3);
\coordinate(u) at (340:1) edge (v3);
\node[fac] at (0,0) {$X$} edge (v1) edge (v2) edge (v3);
\end{tikzpicture}
\quad
\begin{tikzpicture}
\node[draw,cloud] at (0,0) {$H$};
\node[ext](x1) at (90:0.5) {};
\node[ext](x2) at (210:0.5) {};
\node[ext](x3) at (330:0.5) {};
\end{tikzpicture}
\makebox[0.75in]{$\xRightarrow{\text{replace}}$}
\begin{tikzpicture}
\node[draw,cloud] at (0,0) {$H$};
\node[var](v1) at (90:0.5) {};
\node[var](v2) at (210:0.5) {};
\node[var](v3) at (330:0.5) {};
\coordinate(u) at (80:1) edge (v1);
\coordinate(u) at (100:1) edge (v1);
\coordinate(u) at (200:1) edge (v2);
\coordinate(u) at (220:1) edge (v2);
\coordinate(u) at (320:1) edge (v3);
\coordinate(u) at (340:1) edge (v3);
\end{tikzpicture}
&
\begin{tikzpicture}[baseline=-1cm]
\node(n) at (0,0) {$\pi$};
\coordinate(n1) at (-1,-1);
\node at (0,-1) {$\cdots$};
\coordinate(nk) at (1,-1);
\draw (n.south) to node[auto=right] {\small $e_1$} (n1.north);
\draw (n.south) to node[auto=left] {\small $e_k$} (nk.north);
\draw (n1.south) -- +(-0.5cm,-0.75cm) -- +(0.5cm,-0.75cm) -- cycle;
\node at ($(n1.south)+(0,-0.5cm)$) {$D_1$};
\draw (nk.south) -- +(-0.5cm,-0.75cm) -- +(0.5cm,-0.75cm) -- cycle;
\node at ($(nk.south)+(0,-0.5cm)$) {$D_k$};
\end{tikzpicture}
\\
\\
(a) & (b)
\end{tabular}
\end{center}
\caption{(a) Example of replacing a hyperedge labeled $X$ with a hypergraph fragment $H$. Here $|\type(X)|=3$, but in general, there could be any number of endpoint/external nodes, including zero. (b) A derivation tree.}
\label{fig:replacement}
\end{figure}
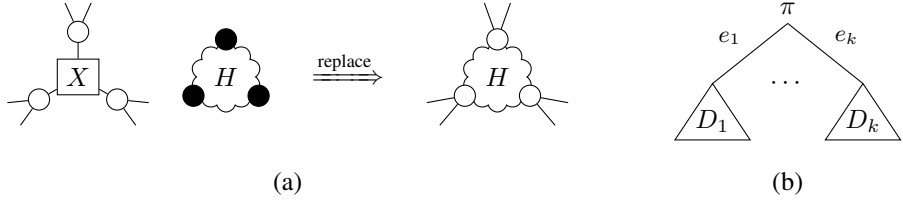

\begin{definition}
Let $G$ be a HRG. For all nonterminals $X$, define the set $\mathcal{D}(G, X)$ of \emph{$X$-type derivation trees} (or simply \emph{$X$-type derivations}) of $G$ to be the smallest set containing all 
finite, unordered, edge-labeled trees of the form shown in Figure~\ref{fig:replacement}b,
where $\pi = (X\rightarrow R)$ is a rule in $G$, $R$ has nonterminal-labeled edges $e_1, \ldots, e_k$ with labels $X_1, \ldots, X_k$, and for $i=1, \ldots, k$, $D_i$ is an $X_i$-type derivation. We simply write \emph{derivation} for $S$-type derivation, and we let $\mathcal{D}(G) = \mathcal{D}(G, S)$.

The \emph{derived graph} of a derivation $D$ is the graph formed as follows. If $D$ is as shown in Figure~\ref{fig:replacement}b, then for $i = 1, \ldots k$, let $H_i$ be the derived graph of $D_i$. In (a copy of) $R$, replace $e_i$ with $H_i$, making the $j$th endpoint of $e_i$ and the $j$th external node of $H_i$ into the same node (for $j=1, \ldots, |\att(e_i)|$). The resulting node is external iff the $j$th endpoint was. All other nodes are kept distinct. (Again, see Figure~\ref{fig:replacement}a for an example with $X=X_i$ and $H=H_i$.)
\end{definition}

From now on, when we mention a derivation $\deriv$ in a context where a graph would be expected, the derived graph of $\deriv$ is to be understood. 

\section{Factor Graph Grammars}
\label{sec:fgg}

\begin{definition}
A HRG for factor graphs, or a \emph{factor graph grammar} (FGG) for short, is a HRG together with mappings $\domain$ and $F$, as in the definition of factor graphs (Definition \ref{def:factorgraph}), except that $F$ is defined on terminal edge labels only.
\end{definition}
%We draw variables and terminal edges the same way they are drawn in factor graphs, and nonterminal edges the same way they are drawn in HRGs.

\begin{example} \label{eg:hmm} Figure~\ref{fig:fgg_hmm} shows a FGG which is equivalent to a HMM. It generates an infinite number of graphs, one for each string length. Also shown is the derivation tree of the factor graph of Example 3.
\end{example}

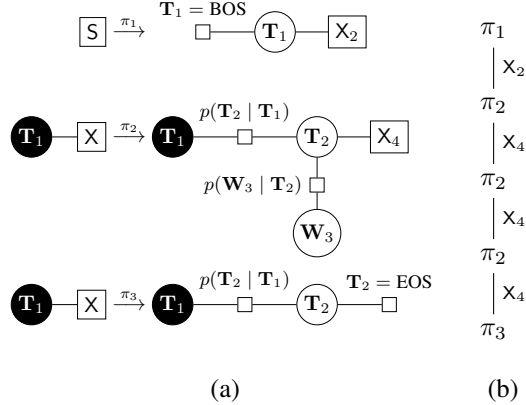
\begin{figure}
\begin{center}
\begin{tabular}{cc}
\scalebox{0.8}{%
$\begin{aligned}[t]
\begin{tikzpicture} 
\node[fac] at (0,0) { $\nt{S}$ }; 
\end{tikzpicture} 
&\xlongrightarrow{\pi_1}
\begin{tikzpicture}[x=1.2cm,y=0.8cm]
\node[var] (t) at (0,0) {$\rv{T}_1$}; 
\node[fac,label=above:{$\rv{T}_1 = \text{BOS}$}] at (-1,0) {} edge (t); 
\node[fac] at (1,0) {$\nt{X}_2$} edge (t); 
\end{tikzpicture} 
\\
\\[1ex]
\begin{tikzpicture} 
\node[ext] (x) at (-1,0) {$\rv{T}_1$}; \node[fac] at (0,0) { $\nt{X}$ } edge (x); 
\end{tikzpicture} 
&\xlongrightarrow{\pi_2}
\begin{tikzpicture}[x=1.2cm,y=0.8cm] 
\node[ext] (x) at (-2,0) {$\rv{T}_1$}; 
\node[var] (t) at (0,0) {$\rv{T}_2$}; 
\node[var] (w) at (0,-2) {$\rv{W}_3$};
\node[fac,label=above:{$p(\rv{T}_2\mid \rv{T}_1)$}] at (-1,0) {} edge (x) edge (t); 
\node[fac] at (1,0) {$\nt{X}_4$} edge (t); 
\node[fac,label=left:{$p(\rv{W}_3 \mid \rv{T}_2)$}] at (0,-1) {} edge (t) edge (w);
\end{tikzpicture} 
\\
\begin{tikzpicture} 
\node[ext] (x) at (-1,0) {$\rv{T}_1$}; \node[fac] at (0,0) { $\nt{X}$ } edge (x); 
\end{tikzpicture} 
&\xlongrightarrow{\pi_3}
\begin{tikzpicture}[x=1.2cm,y=0.8cm] 
\node[ext] (t1) at (-2,0) {$\rv{T}_1$}; 
\node[var] (t2) at (0,0) {$\rv{T}_2$}; 
\node[fac,label=above:{$p(\rv{T}_2 \mid \rv{T}_1)$}] at (-1,0) {} edge (t1) edge (t2);
\node[fac,label=above:{$\rv{T}_2 = \text{EOS}$}] at (1,0) {} edge (t2); 
\end{tikzpicture}
\end{aligned}$}
&
\begin{tikzpicture}
\tikzset{site/.style={auto=right,font={\scriptsize}}}
\node (d) {$\pi_1$};
\node[below of=d] (d1) {$\pi_2$} edge node[site] {$\nt{X}_2$} (d);
\node[below of=d1] (d11) {$\pi_2$} edge node[site] {$\nt{X}_4$} (d1);
\node[below of=d11] (d111) {$\pi_2$} edge node[site] {$\nt{X}_4$} (d11);
\node[below of=d111] (d1111) {$\pi_3$} edge node[site] {$\nt{X}_4$} (d111);
\end{tikzpicture}
\\
\\
(a) & (b)
\end{tabular}
\end{center}
\caption{(a) A FGG generating the infinite set of unrollings of a HMM, one for each sequence length. Each rule is labeled $\pi_i$ for use in the derivation tree. (b) Derivation tree of the factor graph of Example~\ref{eg:hmm3}. An edge from parent $\pi$ with label $X$ to child $\pi'$ means that the right-hand side of $\pi'$ replaces the edge labeled $X$ in the right-hand side of $\pi$.}
\label{fig:fgg_hmm}
\end{figure}

Example~\ref{eg:pcfg} in Appendix~\ref{sec:pcfg} shows how to simulate a PCFG in Chomsky normal form as a FGG.

The graphs generated by a FGG can be viewed, together with $\domain$ and $F$, as factor graphs, each of which defines a (not necessarily normalized) distribution over assignments. Moreover, the whole language of the FGG defines a (not necessarily normalized) distribution over derivations and assignments to the variables in them. If $\deriv \in \mathcal{D}(G)$, then
\[ w_G(\deriv, \xi) = w_\deriv(\xi).\]

FGGs can simulate several other formalisms for dynamically-structured models. As mentioned above (\S\ref{sec:intro}), they can solve two problems that previous formalisms have addressed separately.

FGGs can generate repeated substructures like plate notation \citep{buntine:jair1994,obermeyer+:icml2019} and dynamic graphical models \citep{bilmes:2010} can. There are some structures that plate notation can describe that a FGG cannot -- like the set of all restricted Boltzmann machines, which have two fully-connected layers of nodes. But these are the same structures that \citet{obermeyer+:icml2019} try to avoid, because inference on them is (believed) intractable. FGG rules these structures out naturally.

FGGs can generate alternative substructures like case--factor diagrams \citep{mcallester+:2008} and valid sum--product networks \citep{poon+domingos:2011} can; in particular, they can simulate PCFGs in such a way that inference is equivalent to the cubic-time inside and Viterbi algorithms.

\begin{theorem}
All of the following can be converted into an equivalent FGG:
\begin{compactenum}
\item Plated factor graphs for which the sum--product algorithm of \citet{obermeyer+:icml2019} succeeds.
\item Dynamic graphical models.
\item Case--factor diagrams.
\item Valid sum--product networks.
\end{compactenum}
\end{theorem}
\begin{proof}
See Appendix~\ref{sec:otherformalisms}.
\end{proof}

\section{Conjunction}
\label{sec:conjunction}

The preceding examples show how to use FGGs to model the probability of all tagged strings or all trees generated by a grammar. But it's common for queries to constrain some variables to fixed values, sum over some variables, and get the distribution of the remaining variables. How do such queries generalize to FGGs? For example, in a HMM, how do we compute the probability of all taggings of a given string? Or, how do we compute the marginal distribution of the second-to-last tag? 

To answer such questions, we need to be able to specify a set of nodes across the graphs of a graph language, like the second-to-last tag. Our only means of doing this is to specify a particular node in a particular right-hand side, which could correspond to zero, one, or many nodes in the derived graphs. And we can modify a FGG so that a particular node in a particular right-hand side is always (say) the second-to-last tag. But we propose to factor such modifications into a separate FGG, keeping the FGG describing the model unchanged. Then the modifications can be applied using a \emph{conjunction} operation, which we describe in this section.

Conjunction is closely related to synchronous HRGs \citep{jones+:coling2012}, and, because HRG derivation trees are generated by regular tree grammars, to intersection/composition of finite tree automata/transducers \citep{comon+:2007}. It is also similar to the \textsf{\small PRODUCT} operation on weighted logic programs \citep{cohen+:2010}.

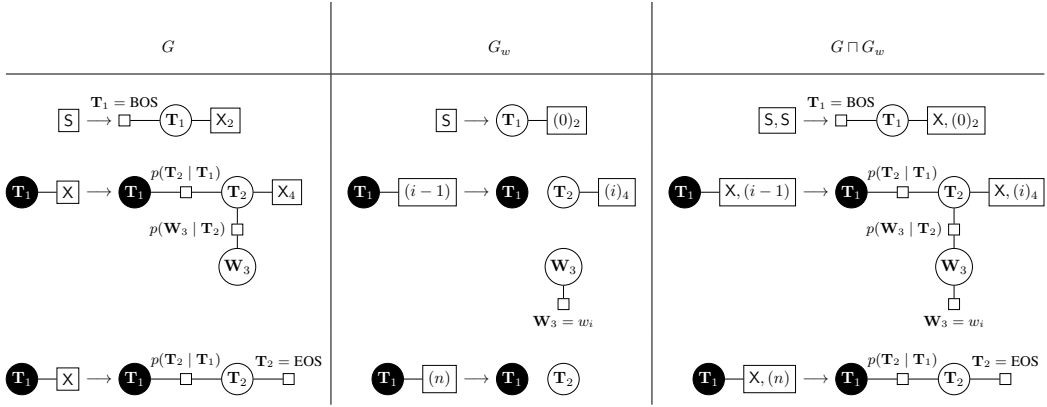
\begin{figure}
\resizebox{\textwidth}{!}{%
\tikzset{x=1.1cm,y=0.8cm,node distance=0.4cm}
\renewcommand{\arraycolsep}{2pt}
\renewcommand{\arraystretch}{4}
$\begin{array}{@{}rcl|@{\quad}rcl@{\quad}|@{\quad }rcl@{}}
\multicolumn{3}{c|@{\quad}}{G} &
\multicolumn{3}{c|@{\quad}}{G_w} &
\multicolumn{3}{c@{}}{G \conj G_w} \\
\hline
\begin{tikzpicture} 
\node[fac] at (0,0) { $\nt{S}$ }; 
\end{tikzpicture} 
&\longrightarrow&
\begin{tikzpicture}
\node[var] (t) at (0,0) {$\rv{T}_1$}; 
\node[fac,label=above:{$\mathclap{\rv{T}_1 = \text{BOS}}$}] at (-1,0) {} edge (t); 
\node[fac,right=of t] {$\nt{X}_2$} edge (t); 
\end{tikzpicture} &
\begin{tikzpicture} 
\node[fac] at (0,0) { $\nt{S}$ }; 
\end{tikzpicture} 
&\longrightarrow&
\begin{tikzpicture} 
\node[var] (t) at (0,0) {$\rv{T}_1$}; 
\node[fac,right=of t] { $\pos{0}_2$ } edge (t); 
\end{tikzpicture}
&
\begin{tikzpicture} 
\node[fac] at (0,0) { $\nt{S,S}$ }; 
\end{tikzpicture} 
&\longrightarrow&
\begin{tikzpicture}
\node[var] (t) at (0,0) {$\rv{T}_1$}; 
\node[fac,label=above:{$\mathclap{\rv{T}_1 = \text{BOS}}$}] at (-1,0) {} edge (t); 
\node[fac,right=of t] {$\nt{X},\pos{0}_2$} edge (t); 
\end{tikzpicture} 
\\
\begin{tikzpicture} 
\node[ext] (x) at (-0.5,0) {$\rv{T}_1$}; 
\node[fac,right=of x] { $\nt{X}$ } edge (x); 
\end{tikzpicture} 
&\longrightarrow&
\begin{tikzpicture}
\node[ext] (x) at (-2,0) {$\rv{T}_1$}; 
\node[var] (t) at (0,0) {$\rv{T}_2$}; 
\node[var] (w) at (0,-2) {$\rv{W}_3$};
\node[fac,label=above:{$p(\rv{T}_2\mid \rv{T}_1)$}] at (-1,0) {} edge (x) edge (t); 
\node[fac,right=of t] {$\nt{X}_4$} edge (t); 
\node[fac,label=left:{$p(\rv{W}_3 \mid \rv{T}_2)$}] at (0,-1) {} edge (t) edge (w);
\end{tikzpicture} &
\begin{tikzpicture} 
\node[ext] (x) at (-0.5,0) {$\rv{T}_1$};
\node[fac,right=of x] { $\pos{i-1}$ } edge (x); 
\end{tikzpicture} 
&\longrightarrow&
\begin{tikzpicture}
\node[ext] (x) at (-1,0) {$\rv{T}_1$}; 
\node[var] (t) at (0,0) {$\rv{T}_2$}; 
\node[var] (w) at (0,-2) {$\rv{W}_3$};
\node[fac,label=below:{\makebox[0pt]{$\rv{W}_3=w_i$}}] at (0,-3) {} edge (w);
\node[fac,right=of t] {$\pos{i}_4$} edge (t);
\end{tikzpicture}
&
\begin{tikzpicture}
\node[ext] (x) at (-0.5,0) {$\rv{T}_1$}; 
\node[fac,right=of x] { $\nt{X},\pos{i-1}$ } edge (x); 
\end{tikzpicture} 
&\longrightarrow&
\begin{tikzpicture}
\node[ext] (x) at (-2,0) {$\rv{T}_1$}; 
\node[var] (t) at (0,0) {$\rv{T}_2$}; 
\node[var] (w) at (0,-2) {$\rv{W}_3$};
\node[fac,label=below:{$\rv{W}_3=w_i$}] at (0,-3) {} edge (w);
\node[fac,label=above:{$p(\rv{T}_2\mid \rv{T}_1)$}] at (-1,0) {} edge (x) edge (t); 
\node[fac,right=of t] {$\nt{X},\pos{i}_4$} edge (t); 
\node[fac,label=left:{$p(\rv{W}_3 \mid \rv{T}_2)$}] at (0,-1) {} edge (t) edge (w);
\end{tikzpicture} \\
\begin{tikzpicture}
\node[ext] (x) at (-0.5,0) {$\rv{T}_1$}; 
\node[fac,right=of x] { $\nt{X}$ } edge (x); 
\end{tikzpicture} 
&\longrightarrow&
\begin{tikzpicture}
\node[ext] (t1) at (-2,0) {$\rv{T}_1$}; 
\node[var] (t2) at (0,0) {$\rv{T}_2$}; 
\node[fac,label=above:{$p(\rv{T}_2 \mid \rv{T}_1)$}] at (-1,0) {} edge (t1) edge (t2);
\node[fac,label=above:{$\rv{T}_2 = \text{EOS}$}] at (1,0) {} edge (t2); 
\end{tikzpicture}
&
\begin{tikzpicture} 
\node[ext] (x) at (-0.5,0) {$\rv{T}_1$};
\node[fac,right=of x] { $\pos{n}$ } edge (x); 
\end{tikzpicture} 
&\longrightarrow&
\begin{tikzpicture}
\node[ext] (t1) at (-1,0) {$\rv{T}_1$}; 
\node[var] (t2) at (0,0) {$\rv{T}_2$}; 
\end{tikzpicture}
&
\begin{tikzpicture}
\node[ext] (x) at (-0.5,0) {$\rv{T}_1$}; 
\node[fac,right=of x] { $\nt{X},\pos{n}$ } edge (x); 
\end{tikzpicture} 
&\longrightarrow&
\begin{tikzpicture}
\node[ext] (t1) at (-2,0) {$\rv{T}_1$}; 
\node[var] (t2) at (0,0) {$\rv{T}_2$}; 
\node[fac,label=above:{$p(\rv{T}_2 \mid \rv{T}_1)$}] at (-1,0) {} edge (t1) edge (t2);
\node[fac,label=above:{$\rv{T}_2 = \text{EOS}$}] at (1,0) {} edge (t2); 
\end{tikzpicture}
\end{array}$%
}
\caption{Illustration of the conjunction operation (Example~\ref{eg:conjoinfgg}). In rules with $i$ in their nonterminals, $i$ ranges from 1 to $n$ where $n = |w|$.}
\label{fig:conjoinfgg}
\end{figure}

\begin{definition}
Two FGG rules are \emph{conjoinable} if they can be written in the form
\begin{align*}
X_1 &\rightarrow R_1 & R_1 &= (V, E_N \cup E_1, \att_N \cup \att_1, \lab^V, \lab^E_1, \ext) \\
X_2 &\rightarrow R_2 & R_2 &= (V, E_N \cup E_2, \att_N \cup \att_2, \lab^V, \lab^E_2, \ext),
\end{align*}
where
\begin{compactitem}
\item $E_N$ contains only nonterminal edges, and $\att_N$ is defined on $E_N$.
\item $E_1$, $E_2$ contain only terminal edges, and $\att_1$, $\att_2$ are defined on $E_1$, $E_2$, respectively.
\item $\type(X_1) = \type(X_2)$, and for $e \in E_N$, $\type(\lab^E_1(e)) = \type(\lab^E_2(e))$.
\end{compactitem}
Then their \emph{conjunction} is
\begin{align*}\langle X_1, X_2\rangle &\rightarrow R & R = (V, E_N \cup E_1 \uplus E_2, \att, \lab^V, \lab^E, \ext)\end{align*} where $\uplus$ means that all edges in $E_1$ and $E_2$ are kept distinct while taking their union, and
\begin{align*}
\lab^E(e) &= \begin{cases}
\langle\lab^E_1(e), \lab^E_2(e)\rangle & \text{if $e \in E_N$} \\
\lab^E_1(e) & \text{if $e \in E_1$} \\
\lab^E_2(e) & \text{if $e \in E_2$}
\end{cases} &
\att(e) &= \begin{cases}
\att_N(e) & \text{if $e \in E_N$} \\
\att_1(e) & \text{if $e \in E_1$} \\
\att_2(e) & \text{if $e \in E_2$}.
\end{cases}
\end{align*}
%Two FGGs $G_1$ and $G_2$ are \emph{conjoinable} if every pair of rules with the same rule label are conjoinable. 
The \emph{conjunction} of two FGGs $G_1$ and $G_2$, written as $G_1 \sqcap G_2$, is the FGG containing the conjunction of all conjoinable pairs of rules from $G_1$ and $G_2$.
\end{definition}

\begin{example} \label{eg:conjoinfgg}
Our FGG for HMMs (Example~\ref{eg:hmm}) is repeated in Figure~\ref{fig:conjoinfgg} as $G$.
We can constrain the $\rv{W}$ variables to an observed string $w$ using another FGG, $G_w$, which has the same variables as $G$ but different factors; its nonterminal edges are the same as $G$ but with different labels. This FGG generates just one graph, whose $\rv{W}$ nodes spell out the string $w$. The conjunction of these two FGGs is shown in the last column ($G \conj G_w$). It combines the factors and nonterminal labels of $G$ and $G_w$ and generates just one graph, the HMM for string $w$.
\end{example}

\begin{example}
To compute the distribution of the second-to-last tag, we need a way of identifying the variable for the second-to-last tag across all graphs. We can do this by conjoining with the FGG:
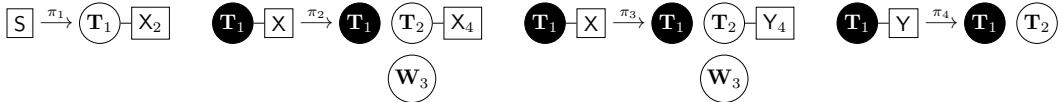
\shrinkdisplay{\tikzset{node distance=0.18cm}
\begin{align*}
\begin{tikzpicture}
\node[fac] {$\nt{S}$};
\end{tikzpicture}
&\xlongrightarrow{\pi_1}
\begin{tikzpicture}
\node[var] (t1) {$\rv{T}_1$};
\node[fac,right=of t1] {$\nt{X}_2$} edge (t1);
\end{tikzpicture} 
&\quad
\begin{tikzpicture}
\node[ext] (t1) {$\rv{T}_1$};
\node[fac,right=of t1] {$\nt{X}$} edge (t1);
\end{tikzpicture}
&\xlongrightarrow{\pi_2}
\begin{tikzpicture}
\node[ext] (t1) {$\rv{T}_1$};
\node[var,right=of t1] (t2) {$\rv{T}_2$};
\node[fac,right=of t2] {$\nt{X}_4$} edge (t2);
\node[var,below=of t2] {$\rv{W}_3$};
\end{tikzpicture}
&\quad\begin{tikzpicture}
\node[ext] (t1) {$\rv{T}_1$};
\node[fac,right=of t1] {$\nt{X}$} edge (t1);
\end{tikzpicture}
&\xlongrightarrow{\pi_3}
\begin{tikzpicture}
\node[ext] (t1) {$\rv{T}_1$};
\node[var,right=of t1] (t2) {$\rv{T}_2$};
\node[fac,right=of t2] {$\nt{Y}_4$} edge (t2);
\node[var,below=of t2] {$\rv{W}_3$};
\end{tikzpicture}
&\quad
\begin{tikzpicture}
\node[ext] (t1) {$\rv{T}_1$};
\node[fac,right=of t1] {$\nt{Y}$} edge (t1);
\end{tikzpicture}
&\xlongrightarrow{\pi_4}
\begin{tikzpicture}
\node[ext] (t1) {$\rv{T}_1$};
\node[var,right=of t1] (t2) {$\rv{T}_2$};
\end{tikzpicture}
\end{align*}}
Then the second-to-last tag is always node $\rv{T}_1$ in the right-hand side of rule $\pi_3$. The methods of the following section can then be used to compute the distribution of this node.
\end{example}

Example~\ref{eg:pcfg_conjoin} in Appendix~\ref{sec:pcfg} shows how to use conjunction to constrain a PCFG to a single input string.
\section{Inference}
\label{sec:inference}

Given a FGG $G$, we want to be able to efficiently compute its sum--product,
\[ Z_G = \sum_{\deriv \in \mathcal{D}(G)} \sum_{\asst \in \Asst_\deriv} w_G(\deriv, \asst). \]
We can answer a wide variety of queries by using the conjunction operation to constrain variables based on observations, and then computing the sum--product in various semirings: ordinary addition and multiplication would sum over assignments to the remaining variables, and the expectation semiring \citep{eisner:2002} would compute expectations with respect to them. The Viterbi (max--product) semiring would find the highest-weight \emph{derivation} and assignment, not necessarily the highest-weight \emph{graph} and assignment, which is NP-hard \citep{lyngso+pederson:2002}.

We consider three cases below: 
finite variable domains, but possibly infinite graph languages (\S\ref{sec:inference_cfg}); finite graph languages, but possibly infinite variable domains (\S\ref{sec:inference_graph}); 
and infinite variable domains and graph languages (\S\ref{sec:inference_undecidable}).
To help characterize these cases and their subcases, we introduce the following definitions.
\begin{definition}
A FGG is \emph{recursive} if it has an $X$-type derivation that contains an $X$-type derivation as a proper subtree; otherwise, it is \emph{nonrecursive}. A nonrecursive FGG generates a finite set of graphs; this is a common case, because the conjunction of any FGG with a nonrecursive FGG (e.g., one describing a finite-sized observation) is nonrecursive. 

A recursive FGG is \emph{nonlinearly recursive} if it has an $X$-type derivation that contains two disjoint $X$-type derivations as proper subtrees; otherwise, it is \emph{linearly recursive}.

A FGG is \emph{nonreentrant} if no derivation contains two different $X$-type derivations as subtrees. Every nonreentrant FGG is nonrecursive, and any nonrecursive FGG can be made nonreentrant by duplicating rules and renaming nonterminals (though this may cause an exponential blowup in the size of the grammar).
\end{definition}

\subsection{Finite variable domains}
\label{sec:inference_cfg}

When a HRG generates a graph, the derivation tree is isomorphic to a tree decomposition of the graph: each derivation tree node $\pi = (X \rightarrow R)$ corresponds to a bag of the tree decomposition containing the nodes in $R$. It follows that a HRG whose right-hand sides have at most $(k+1)$ nodes generates graphs with treewidth at most $k$ \cite[Theorem 37]{bodlaender:1998}. So if a FGG $G$ generates a graph $H$, computing the sum--product of $H$ by variable elimination (VE) takes time linear in the size of $H$ and exponential in $k$.

In this section, we generalize VE to compute the sum-product of all graphs generated by $G$ without enumerating them. If $G$ is nonrecursive, this is (like VE) linear in the size of $G$ and exponential in $k$; in the envisioned typical use-case, we have a fixed FGG $G$ representing the model and different FGGs $G'$ representing different observations; since conjunction cannot increase $k$, we may regard $k$ as fixed, so computing the sum--product of $G \conj G'$ takes time linear in the size of $G \conj G'$.

\begin{theorem} \label{thm:inference_cfg}
Let $G = (N, T, P, S)$ be a FGG such that for all $v$ in $G$, $|\domain(v)| \leq m$. Let $|G|$ be the number of rules in $G$, and let $k$ be such that every right-hand side in $G$ has at most $(k+1)$ nodes. Then $Z_G$ is the least solution of a monotone system of polynomial equations, and in particular:
\begin{compactenum}
\item If $G$ is nonrecursive, $Z_G$ can be computed in $O(|G|m^{k+1})$ time. 
\item If $G$ is linearly recursive, $Z_G$ can be computed in $O(|G|^3 m^{3(k+1)})$ time in the worst case.
\end{compactenum}
\end{theorem}

\begin{proof} 
The computation of the sum--product is closely analogous to the sum--product of a PCFG \citep{stolcke:1995,nederhof+satta:2008}. We introduce some shorthand for assignments. If $\asst$ is an assignment and $v_1 \cdots v_l$ is a sequence of nodes, we write $\asst(v_1 \cdots v_l)$ for $\asst(v_1) \cdots \asst(v_l)$. If $X$ is a nonterminal and $\type(X) = \ell_1 \ldots \ell_k$, we define $\Asst_X = \domain(\ell_1) \times \cdots \times \domain(\ell_k)$, the set of assignments to the endpoints of an edge labeled $X$.

Next, we define a system of equations whose solution gives the desired sum--product. The unknowns are $\psi_X(\asst)$ for all $X \in N$ and $\asst \in \Asst_X$, and $\tau_R(\asst)$ for all rules $(X\rightarrow R)$ and $\asst \in \Asst_X$.
For all $X \in N$, let $P^X$ be the rules in $P$ with left-hand side $X$. For each $\asst \in \Asst_X$, add the equation
\begin{equation*}
  \psi_X(\asst) = \sum_{(X \rightarrow R) \in P^X} \tau_R(\asst).
\end{equation*}
For each right-hand side $R = (V, E_N \cup E_T, att, lab^V, lab^E, ext)$, where $E_N$ contains only nonterminal edges and $E_T$ contains only terminal edges, and for each $\asst \in \Asst_X$, add the equation
\begin{equation*}
\tau_R(\asst) = \sum_{\substack{\asst' \in \Asst_R\\\asst'(\ext) = \asst}} \prod_{e \in E_T} F(e)(\asst'(\att(e))) \prod_{e \in E_N} \psi_{\lab^E(e)}(\asst'(\att(e))).
\end{equation*}
Then $\sum_{\asst \in \Asst_X} \psi_X(\asst)$ represents the sum--product of all $X$-type derivations. In particular, the sum--product of the FGG is $\psi_S()$.

To solve these equations, construct a directed graph over nonterminals with an edge from $X$ to $Y$ iff there is a rule $X \rightarrow R$ where $R$ contains an edge labeled $Y$. For each connected component~$C$ of this graph in reverse topological order:
\begin{compactenum}
\item If $C = \{X\}$, compute $\psi_X$ and substitute it into the other equations. \label{item:singleton}
\item Else if the equations for $\psi_X$ and $\tau_R$ where $X\in C$ and $(X\rightarrow R) \in P$ are linear, solve them and substitute into the other equations \citep{stolcke:1995,goodman:1999}. \label{item:linear}
\item Else, the equations can be approximated iteratively \citep{goodman:1999,nederhof+satta:2008}. \label{item:nonlinear}
\end{compactenum}

If $G$ is nonrecursive, the graph of nonterminals is acyclic, so case (\ref{item:singleton}) always applies. The total running time is $O(|G| m^{k+1})$.

If $G$ is linearly recursive, then case (\ref{item:linear}) may also apply. In the worst case, the nonterminal graph is one connected component, corresponding to $O(|G| m^{k+1})$ unknowns. Solving the equations could involve inverting a matrix of this size, which takes $O(|G|^3 m^{3(k+1)})$ time.

If $G$ is nonlinearly recursive, any of the three cases may apply. For case (\ref{item:nonlinear}), each iteration takes $O(|G| m^{k+1})$ time (fixed-point iteration method) or $O(|G|^3 m^{3(k+1)})$ time (Newton's method), but the number of iterations depends on $G$.
\end{proof}

Finally, we note that we can reduce the sizes of the right-hand sides of a FGG by a process analogous to binarization of CFGs \citep{gildea:2011,chiang+:acl2013}:
\begin{proposition} \label{lem:factorize}
For any hypergraph fragment $R$, let $\bar{R}$ be the hypergraph formed by adding a hyperedge connecting $R$'s external nodes. Let $G$ be a HRG, $n_G$ be the total number of nodes in its right-hand sides, and $k$ be such that for every right-hand side $R$, the treewidth of $\bar{R}$ is at most $k$. Then there is an equivalent HRG with at most $n_G$ rules whose right-hand sides have at most $(k+1)$ nodes.
\end{proposition}
\begin{proof}
See Appendix~\ref{app:factorize}.
\end{proof}

\subsection{Finite graph languages}
\label{sec:inference_graph}

Next, we show that a nonrecursive FGG can also be converted into an equivalent factor graph, such that the sum--product of the factor graph is equal to the sum--product of the FGG. This makes it possible to use standard graphical model inference techniques for reasoning about the FGG, even with infinite variable domains. However, the conversion increases treewidth in general, so when the method of Section~\ref{sec:inference_cfg} is applicable, it should be preferred.

The construction is similar to constructions by \citet{smith+eisner:2008} and \citet{pynadath+wellman:1998} for dependency parsers and PCFGs, respectively. Their constructions and ours encode a set of possible derivations as a graphical model, using hard constraints to ensure that every assignment to the variables corresponds to a valid derivation.

\begin{theorem} \label{thm:inference_graph}
Let $G = (N, T, P, S)$ be a nonreentrant FGG. 
% Let $k$ be the maximum arity of any nonterminal.
Let $n_G$ and $m_G$ be the total number of nodes and edges in the right-hand sides of $G$ respectively. Then $G$ can be converted into a factor graph with $O(n_G)$ variables and $O(n_G + m_G)$ factors which gives the same sum--product.
% at most $n_G + (k+1)|N| + |P|$ variables, and at most $m_G + n_G + (k+2)|N| + 2k|P|$ factors.
\end{theorem}
\begin{proof}
We construct a factor graph that encodes all derivations of $G$. (Example~\ref{eg:finitegrammar} in Appendix~\ref{app:finitegrammar} shows an example of this construction for a toy FGG.)
First, we add binary variables (with label $\rv{B}$ where $\domain(\rv{B}) = \{\text{true}, \text{false}\}$) that switch on or off parts of the factor graph (somewhat like the gates of \citet{minka+winn:2008}). For each nonterminal $X \in N$, we add $\rv{B}_X$, indicating whether $X$ is used in the derivation, and for each rule $\pi \in P$, we add $\rv{B}_\pi$, indicating whether $\pi$ is used.

Next, we create factors that constrain the $\rv{B}$ variables so that only one derivation is active at a time.
We write $P^X$ for the set of rules with left-hand side $X$, and $P^{\rightarrow X}$ for the set of rules which have a right-hand side edge labeled $X$. 
Define the following function:
\begin{equation*}
    \text{CondOne}_l(\rv{B}, \rv{B}_1, \ldots, \rv{B}_l) = 
        \begin{cases}
            \exists! \, i \in \{1, \ldots, l\} \, . \, \rv{B}_i & \text{if $\rv{B} = \text{true}$} \\
            \neg (\rv{B}_1 \lor \cdots \lor \rv{B}_l) & \text{if $\rv{B} = \text{false}$}
        \end{cases}
\end{equation*}
Then we add these factors, which ensure that if one of the rules in $P^{\rightarrow X}$ is used (or $X=S$), then exactly one rule in $P^X$ is used; if no rule in $P^{\rightarrow X}$ is used (and $X \neq S$), then no rule in $P^X$ is used. %(Since the grammar is nonreentrant, it is guaranteed that at most one $r \in P^{\rightarrow X}$ will have been used.)
\begin{compactitem}
    \item For the start symbol $S$, add a factor $e$ with $att(e) = \rv{B}_S$ and $F(e)(\rv{B}_S) = (\rv{B}_S = \text{true})$.
    \item For $X \in N \setminus \{S\}$, let $P^{\rightarrow X} = \{\pi_1, \ldots, \pi_l\}$ and add a factor $e$ with $\att(e) = \rv{B}_X \, \rv{B}_{\pi_1} \cdots \rv{B}_{\pi_l}$ and $F(e) = \text{CondOne}_l$.    
    \item For $X \in N$, let $P^X = \{\pi_1, \ldots, \pi_l\}$ and add a factor $e$ with $att(e) = \rv{B}_X \, \rv{B}_{\pi_1} \cdots \rv{B}_{\pi_l}$ and $F(e) = \text{CondOne}_l$. 
\end{compactitem}

Next, define the function:
\begin{equation*}
\text{Cond}(\rv{B}, x) = \begin{cases}
x & \text{if $\rv{B} = \text{true}$} \\
1 & \text{otherwise.}
\end{cases}
\end{equation*}

For each rule $\pi \in P$, where $\pi = (X \rightarrow R)$ and $R = (V, E_N \cup E_T, \att, \lab^V, \lab^E)$, we construct a ``cluster'' $C_{\pi}$ of variables and factors:
\begin{compactitem}
    \item For each $v \in V$, add a variable $v'$ with the same label to $C_{\pi}$. Also, add a factor with endpoints $\rv{B}_{\pi}$ and $v'$ and function 
        $\text{CondNormalize}_{v'}(\rv{B}_{\pi}, v')$, defined to equal $\text{Cond}(\neg \rv{B}_\pi, p(v'))$,
    where $p$ is any probability distribution over $\domain(v')$. This ensures that if $\pi$ is not used, then $v'$ will sum out of the sum--product neatly.
    \item For each $e \in E_T$ where $att(e) = v_1 \cdots v_k$, add a new edge $e'$ with $\att(e') = \rv{B}_{\pi}\,v'_{1} \cdots v'_{k}$ and function $\text{CondFactor}_{e'}(\rv{B}_\pi, v_1', \ldots, v_k')$, defined to equal $\text{Cond}(\rv{B}_\pi, F(e)(v_1', \ldots, v_k'))$.
\end{compactitem}

Next, for each $X \in N$, let $l = |\type(X)|$. We create a cluster $C_X$ containing variables $v_{X,i}$ for $i=1, \ldots, l$, which represent the endpoints of $X$, such that $\lab^V(v_{X,i}) = \type(X)_i$. We give each an accompanying factor with endpoints $\rv{B}_\pi$ and  $v_{X,i}$ and function $\text{CondNormalize}_{v_{X,i}}$.

These clusters are used by the factors below, which ensure that if two variables are identified during rewriting, they have the same value. Define $\text{CondEquals}(\rv{B}, v, v') = \text{Cond}(\rv{B}, v=v')$.

\begin{compactitem}
\item For each $\pi \in P^{\rightarrow X}$, let $v_1, \ldots, v_l$ be the endpoints of the edge in $\pi$ labeled $X$. (By non-reentrancy, there can be only one such edge.) For $i = 1, \ldots, l$, create a factor $e$ where $\att(e) = \rv{B}_\pi \, v_{X,i} \, v_i$ and $F(e) = \text{CondEquals}$.
    \item For each $\pi \in P^X$, let $\ext$ be the external nodes of $\pi$. For $i = 1, \ldots, l$, create a factor $e$ where $\att(e) = \rv{B}_\pi \, 
    v_{X,i} \, \ext_i$ and $F(e) = \text{CondEquals}$.
\end{compactitem}

The resulting graph has $|N| + |P|$ binary variables, $n_G$ variables in the clusters $C_{\pi}$, and $\sum_{X \in N} |\type(X)| \leq n_G$ variables in the clusters $C_{X}$, so the total number of variables is in $O(n_G)$. It has $m_G$ $\text{CondFactor}_e$ factors, $n_G + \sum_{X \in N} |\type(X)| \leq 2n_G$ $\text{CondNormalize}_v$ factors, $2|N|$ $\text{CondOne}_l$ factors, and $2\sum_{(X \rightarrow R) \in P} |\ext_R| \leq 2n_G$ $\text{CondEquals}$ factors, so the total number of factors is in $O(n_G + m_G)$.

Appendix~\ref{app:finitegrammar} contains more information on this construction, including an example, a detailed proof that the sum--product is preserved, and a discussion of inference on the resulting graph.
\end{proof}

\subsection{Infinite variable domains, infinite graph languages}
\label{sec:inference_undecidable}

Finally, if we allow both (countably) infinite domains and infinite graph languages, then computing the sum--product is undecidable. This has already been observed even for single factor graphs with infinite variable domains \citep{dreyer+eisner:2009}, but we show further that this can be done using a minimal inventory of factors.

\begin{theorem} \label{thm:undecidable}
Let $G$ be a FGG whose variable domains are $\mathbb{N}$ and whose factors only use the successor relation and equality with zero. It is undecidable whether the sum--product of $G$ is zero.
\end{theorem}

\begin{proof}
By reduction from the halting problem for Turing machines. See Appendix~\ref{sec:tm}.
\end{proof}

\section{Conclusion}

Factor graph grammars are a powerful way of defining probabilistic models that permits practical inference. We plan to implement the algorithms described in this paper as differentiable operations and release them as open-source software. We will also explore techniques for optimizing inference in FGGs, for example, by automatically modifying rules to reduce their treewidth \citep{bilmes:2010} or reducing the cost of matrix inversions in Theorem~\ref{thm:inference_cfg} \citep{nederhof+satta:2008}. Another important direction for future work is the development of approximate inference algorithms for FGGs.

\section*{Broader Impact}

This research is of potential benefit to anyone working with structured probability models, including latent-variable neural networks. As this research is purely theoretical, we are not aware of any direct negative impacts.

\begin{ack}
We would like to thank the anonymous reviewers, especially Reviewer 3, for making numerous suggestions for improvement. We also thank Antonis Anastasopoulos, Justin DeBenedetto, Wes Filardo, Chung-Chieh Shan, and Xing Jie Zhong for their feedback. 

This material is based upon work supported by the National Science Foundation under Grant No.~2019291.  Any opinions, findings, and conclusions or recommendations expressed in this material are those of the authors and do not necessarily reflect the views of the National Science Foundation.
\end{ack}

\bibliography{references}

\clearpage

\appendix

\section{Simulating PCFGs}
\label{sec:pcfg}

\begin{example}
\label{eg:pcfg}
Below is a FGG for derivations of a PCFG in Chomsky normal form. The start symbol of the FGG is $\nt{S'}$ and the start symbol of the PCFG is $S$. Random variables $\rv{N}$ range over nonterminal symbols of the PCFG, and random variables $\rv{W}$ range over terminal symbols.
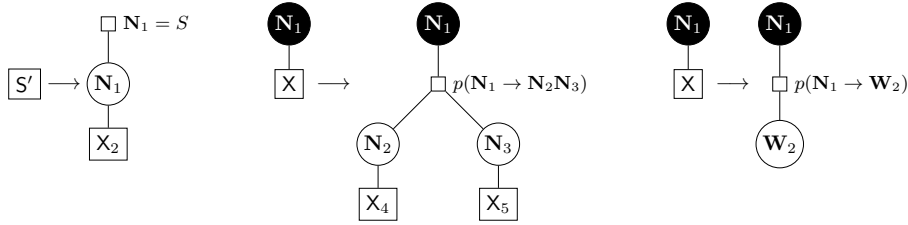
\shrinkdisplay{\begin{align*}
\begin{tikzpicture} 
\node[fac] at (0,0) { $\nt{S'}$ }; 
\end{tikzpicture} 
&\longrightarrow 
\begin{tikzpicture} 
\node[var] (n) at (0,0) { $\rv{N}_1$ }; 
\node[fac,label=right:{$\rv{N}_1 = S$}] at (0,1) {} edge (n);
\node[fac] at (0,-1) {$\nt{X}_2$} edge (n); 
\end{tikzpicture} 
&
\begin{tikzpicture} 
\node[ext](x) at (0,1) {$\rv{N}_1$};
\node[fac] at (0,0) { $\nt{X}$ } edge (x); 
\end{tikzpicture} 
&\longrightarrow 
\begin{tikzpicture} 
\node[ext] (n) at (0,1) { $\rv{N}_1$ }; 
\node[var] (n1) at (-1,-1) {$\rv{N}_2$}; 
\node[var] (n2) at (1,-1) {$\rv{N}_3$}; 
\node[fac,label=right:{$p(\rv{N}_1\rightarrow \rv{N}_2 \rv{N}_3)$}] at (0,0) {} edge (n) edge (n1) edge (n2); 
\node[fac] at (-1,-2) {$\nt{X}_4$} edge (n1);
\node[fac] at (1,-2) {$\nt{X}_5$} edge (n2);
\end{tikzpicture}
&
\begin{tikzpicture} 
\node[ext](x) at (0,1) {$\rv{N}_1$};
\node[fac] at (0,0) { $\nt{X}$ } edge (x); 
\end{tikzpicture} 
&\longrightarrow
\begin{tikzpicture} 
\node[ext] (n) at (0,1) { $\rv{N}_1$ }; 
%\node[fac,label=right:{$\rv{N}_1 = \nt{X}$}] at (1,0) {} edge (n);
\node[var] (n1) at (0,-1) { $\rv{W}_2$ }; \node[fac,label=right:{$p(\rv{N}_1 \rightarrow \rv{W}_2)$}] at (0,0) {} edge (n) edge (n1); 
\end{tikzpicture}
\end{align*}}
\end{example}

\begin{example}
\label{eg:pcfg_conjoin}
We can conjoin the FGG of Example~\ref{eg:pcfg} with the following FGG to constrain it to an input string $w$, with $n = |w|$, $0 \leq i < j < k \leq n$, and $1 \leq l \leq n$:
\shrinkdisplay
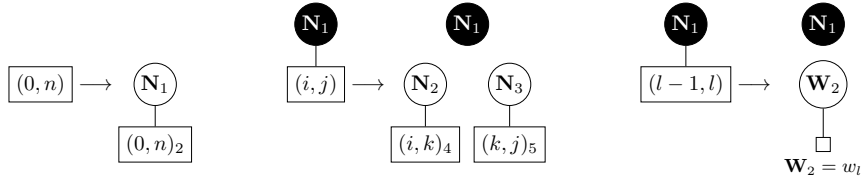
{\begin{align*}
\begin{tikzpicture}
\node[fac] at (0,0) { $\pos{0,n}$ }; 
\end{tikzpicture} &
\longrightarrow
\begin{tikzpicture} 
\node[var](n1) at (0,0) { $\rv{N}_1$ };
\node[fac] at (0,-1) {$\pos{0,n}_2$} edge (n1); 
\end{tikzpicture}
&
\begin{tikzpicture}
\node[ext](x1) at (1,1) {$\rv{N}_1$};
\node[fac] at (1,0) { $\pos{i,j}$ } edge (x1);
\end{tikzpicture} 
&\longrightarrow
\begin{tikzpicture}[x=0.7cm]
\node[ext](n1) at (0,1) {$\rv{N}_1$};
\node[var](n2) at (-1,0) {$\rv{N}_2$};
\node[var](n3) at (1,0) {$\rv{N}_3$};
\node[fac] at (-1,-1) {$\pos{i,k}_4$} edge (n2);
\node[fac] at (1,-1) {$\pos{k,j}_5$} edge (n3);
\end{tikzpicture} 
&
\begin{tikzpicture} 
\node[ext](x1) at (0,1) {$\rv{N}_1$};
\node[fac] at (0,0) { $\pos{l-1,l}$ } edge (x1);
\end{tikzpicture} 
&\longrightarrow
\begin{tikzpicture} 
\node[ext](n1) at (0,1) {$\rv{N}_1$};
\node[var](w2) at (0,0) {$\rv{W}_2$};
\node[fac,label=below:{$\rv{W}_2 = w_l$}] at (0,-1) {} edge (w2);
\end{tikzpicture}
\end{align*}}
\end{example}

The resulting rules have a total of $O(n^3)$ variables in their right-hand sides. The largest right-hand side has 3 variables, so $k=2$. The variables range over nonterminals, so $m = |N|$ where $N$ is the CFG's nonterminal alphabet. Therefore, running the algorithm of Theorem~\ref{thm:inference_cfg} on this FGG takes $O(n_G m^{k+1}) = O(|N|^3 n^3)$ time, which is the same as the CKY algorithm. This construction generalizes easily to CFGs not in Chomsky normal form; applying Lemma~\ref{lem:factorize} would keep the inference complexity down to $O(n^3)$ (or $O(n^2)$ for a linear CFG).
\section{Relationship to other formalisms}
\label{sec:otherformalisms}

\subsection{Plate diagrams}
\label{sec:plate}

Plate diagrams are extensions of graphs that describe repeated structure in Bayesian networks \citep{buntine:jair1994} or factor graphs \citep{obermeyer+:icml2019}. A plate is a subset of variables/factors, together with a count $M$, indicating that the variables/factors inside the plate are to be replicated $M$ times. But there cannot be edges between different instances of a plate.
\begin{definition}
A \emph{plated factor graph} or \emph{PFG} \citep{obermeyer+:icml2019} is a factor graph $H = (V, E)$ together with a finite set $B$ of \emph{plates} and a function $P: V \cup E \rightarrow 2^B$ that assigns each variable and factor to a set of plates. If $b \in P(v)$ and $e$ is incident to $v$, then $b \in P(e)$.

The \emph{unrolling} of $H$ by $M: B \rightarrow \mathbb{N}$  is the factor graph that results from making $M(b)$ copies of every node $v$ such that $b \in P(v)$ and every edge $e$ such that $e \in P(e)$.
\end{definition}

\citet{obermeyer+:icml2019} give an algorithm for computing the sum--product of a PFG. It only succeeds on some PFGs. An example for which it fails is the set of all restricted Boltzmann machines (fully-connected bipartite graphs); %, because the PFG for it has overlapping but non-nested plates. 
one of their main results is to characterize the PFGs for which their algorithm succeeds. Below, we show how to convert these PFGs to FGGs.

\begin{proposition}
Let $H$ be a PFG. If the sum--product algorithm of \citet{obermeyer+:icml2019} succeeds on $H$, then there is a FGG $G$ such that for any $M: B \rightarrow \mathbb{N}$, there is a FGG $G_M$ such that $G \conj G_M$ generates one graph, namely the unrolling of $H$ by $M$.
\end{proposition}

\begin{algorithm}
\begin{algorithmic}
\While{$E \neq \emptyset$}
  \State{let $e = \argmax_{e \in E} |P(e)|$ (breaking ties arbitrarily) and $L = P(e)$}
  \State{let $H_L$ be the subgraph of nodes and edges of $H$ whose plate set is $L$}
  \For{each connected component $H_c$ of $H_L$}
    \State{let $V_f$ be the variables not in $H_c$ but incident to factors in $H_c$}
    \State{let $L' = \cup_{v \in V_f} P(v)$}
    \If{$L=L'$}
      \State{\textbf{error}}
    \EndIf
    \State{let $X$ be a fresh nonterminal}
    \State{let $n = \prod_{b \in L \setminus L'} M(b)$}
    \State{replace $H_c$ with an edge with label $X^n$ and endpoints $V_f$}
    \For{$i \leftarrow n, \ldots, 1$}
      \State{create rule $X^{i} \rightarrow R$ where $R$ has:
      \begin{compactitem}[labelindent=\the\algorithmicleftmargin,leftmargin=!]\item internal nodes and edges from $H_c$ \item external nodes $V_f$ \item an edge with label $X^{i-1}$ and endpoints $V_f$ \end{compactitem}}
    \EndFor
    \State{create rule $X^0 \rightarrow R$ where $R$ has external nodes $V_f$ and no other nodes/edges}
    \EndFor
\EndWhile
\State{create rule $S \rightarrow H$}
\end{algorithmic}
\caption{Procedure for converting a PFG $H$ and count assignment $M$ to a FGG.}
\label{alg:pfg}
\end{algorithm}

\begin{proof}
We just describe how to construct $G \conj G_M$ directly; hopefully, it should be clear how to construct $G$ and $G_M$ separately ($G$ has factors but not counts; $G_M$ has counts but not factors).
Algorithm~\ref{alg:pfg} converts $H$ and $M$ to $G \conj G_M$. It has the same structure as the sum--product algorithm of \citet{obermeyer+:icml2019} and therefore works on the same class of PFGs. 
\end{proof}

If the algorithm of \citet{obermeyer+:icml2019} fails on a PFG, there might not be an equivalent FGG. In particular, FGGs cannot generate the set of RBMs, because a $m \times n$ RBM has treewidth $\min(m,n)$, so the set of all RBMs has unbounded treewidth and can't be generated by a HRG. 

Although, in this respect, FGGs are less powerful than PFGs, we view this as a strength, not a weakness. Because FGGs inherently generate graphs of bounded treewidth, our sum--product algorithm (Theorem~\ref{thm:inference_cfg}) works on all FGGs, and no additional constraints are needed to guarantee efficient inference.

\begin{example}
The following PFG is from \citet{obermeyer+:icml2019}:
\begin{center}
\begin{tikzpicture}[x=1.33cm,y=1.33cm]
\node[var](x) at (1,0) {$\rv{X}$};
\node[var](y) at (3,0) {$\rv{Y}$};
\node[fac,label=above:{$F$}] at (0,0) {} edge (x);
\node[fac,label=above:{$H$}] at (2,0) {} edge (x) edge (y);
\node[fac,label=above:{$G$}] at (4,0) {} edge (y);
\draw[plate] (1.5,-0.6) rectangle (4.5,0.6); \node[anchor=south east] at (4.5,-0.6) {$I$};
\draw[plate] (1.6,-0.5) rectangle (2.5,0.5); \node[anchor=south east] at (2.5,-0.5) {$J$};
\end{tikzpicture}
\end{center}

Converting to a FGG produces the following rules (in order of their construction by the above algorithm):
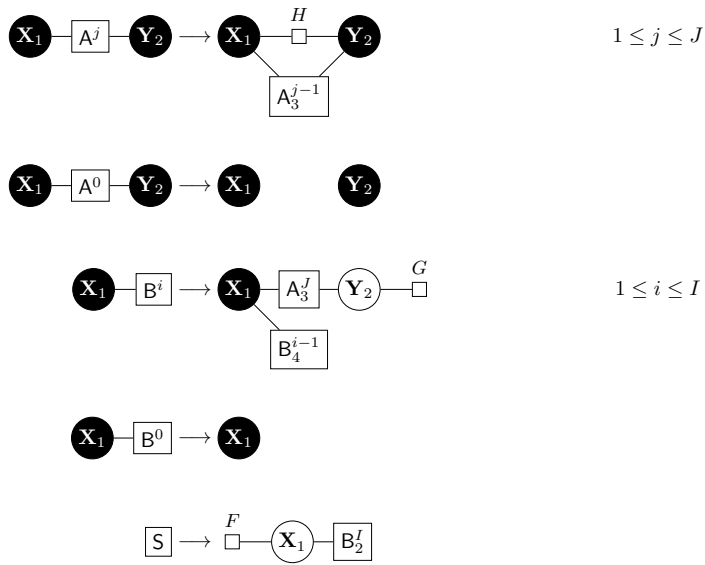
\shrinkdisplay{\begin{align*}
\begin{tikzpicture}
\node[ext](x) at (-1,0) {$\rv{X}_1$};
\node[ext](y) at (1,0) {$\rv{Y}_2$};
\node[fac] {$\nt{A}^j$} edge (x) edge (y);
\end{tikzpicture}
&\longrightarrow
\begin{tikzpicture}
\node[ext](x) at (-1,0) {$\rv{X}_1$};
\node[ext](y) at (1,0) {$\rv{Y}_2$};
\node[fac,label=above:{$H$}] at (0,0) {} edge (x) edge (y);
\node[fac] at (0,-1) {$\nt{A}^{j-1}_3$} edge (x) edge (y);
\end{tikzpicture}
&1 \leq j &\leq J 
\\
\\[1ex]
\begin{tikzpicture}
\node[ext](x) at (-1,0) {$\rv{X}_1$};
\node[ext](y) at (1,0) {$\rv{Y}_2$};
\node[fac] {$\nt{A}^0$} edge (x) edge (y);
\end{tikzpicture}
&\longrightarrow
\begin{tikzpicture}
\node[ext](x) at (-1,0) {$\rv{X}_1$};
\node[ext](y) at (1,0) {$\rv{Y}_2$};
\end{tikzpicture}
\\
\\[1ex]
\begin{tikzpicture}
\node[ext](x) at (-1,0) {$\rv{X}_1$};
\node[fac] {$\nt{B}^i$} edge (x);
\end{tikzpicture}
&\longrightarrow
\begin{tikzpicture}
\node[ext](x) at (-1,0) {$\rv{X}_1$};
\node[var](y) at (1,0) {$\rv{Y}_2$};
\node[fac] at (0,0) {$\nt{A}^J_3$} edge (x) edge (y);
\node[fac] at (0,-1) {$\nt{B}^{i-1}_4$} edge (x);
\node[fac,label=above:{$G$}] at (2,0) {} edge (y);
\end{tikzpicture} 
&1 \leq i &\leq I
\\
\\[1ex]
\begin{tikzpicture}
\node[ext](x) at (-1,0) {$\rv{X}_1$};
\node[fac] {$\nt{B}^0$} edge (x);
\end{tikzpicture}
&\longrightarrow
\begin{tikzpicture}
\node[ext](x) at (-1,0) {$\rv{X}_1$};
\end{tikzpicture}
\\
\\[1ex]
\begin{tikzpicture}
\node[fac] {$\nt{S}$};
\end{tikzpicture}
&\longrightarrow
\begin{tikzpicture}
\node[var](x) at (-1,0) {$\rv{X}_1$};
\node[fac] at (0,0) {$\nt{B}^I_2$} edge (x);
\node[fac,label=above:{$F$}] at (-2,0) {} edge (x);
\end{tikzpicture}
\end{align*}}
\end{example}

\subsection{Dynamic graphical models}
\label{sec:dgm}

For simplicity, we only consider binary factors, which we draw as directed edges, and we ignore edge labels.
\begin{definition}
A \emph{dynamic graphical model} or \emph{DGM} \citep{bilmes:2010} is a tuple $(H_1, H_2, H_3, E_{12}, E_{22}, E_{23})$, where the $H_i = (V_i, E_i)$ are factor graphs and the $E_{ij} \subseteq V_i \times V_j$ are sets of edges from $H_i$ to $H_j$.

A DGM specifies how to construct, for any length $n \geq 2$, a factor graph \[H^n = (V_1 \cup V_2 \times \{1, \ldots, n\} \cup V_3, E),\] where $E$ is defined by:
\begin{compactitem}
\item If $(u,v) \in E_{12}$, add an edge from $u$ to $(v,1)$.
\item If $(u,v) \in E_{22}$, add an edge from $(u,i-1)$ to $(v,i)$ for all $1 < i \leq n$.
\item If $(u,v) \in E_{23}$, add an edge from $(u,n)$ to $v$.
\end{compactitem}
\end{definition}

\begin{proposition}
Given a DGM $D = (H_1, H_2, H_3, E_{12}, E_{22}, E_{23})$, there is a FGG $G$ such that for any count $n \geq 2$, there is another FGG $G_n$ such that $G \conj G_n$ generates exactly one graph, the unrolling of $D$ by $n$.
\end{proposition}
\begin{proof}
Again, we give an algorithm for constructing $G \conj G_n$, and hopefully, it should be clear how to construct $G$ and $G_n$ separately.
Create the following rules:
\begin{compactitem}
\item $\nt{S} \rightarrow R$, where $R$ contains
\begin{compactitem}
\item Nodes and edges from $H_1$, $H_2$, and $E_{12}$
\item An edge labeled $\nt{A}^{n-1}$ and endpoints $\{u \mid (u,v) \in E_{22}\}$.
\end{compactitem}
\item $\nt{A}^{i} \rightarrow R$, where $R$ contains
\begin{compactitem}
\item Nodes and edges from $H_2$
\item If $(u,v) \in E_{22}$, $R$ has an external node $u'$
\item For each $(u,v) \in E_{22}$, an edge from $u'$ to $v$ 
\item An edge labeled $\nt{A}^{i-1}$ and endpoints $\{u \mid (u,v) \in E_{22}\}$.
\end{compactitem}
\item $\nt{A}^1 \rightarrow R$, where $R$ contains
\begin{compactitem}
\item Nodes and edges from $H_2$, $H_3$, and $E_{23}$
\item If $(u,v) \in E_{22}$, $R$ has an external node $u'$
\item For each $(u,v) \in E_{22},$ an edge from $u'$ to $v$.
\end{compactitem}
\end{compactitem}
\end{proof}

\begin{example}
\citet{bilmes:2010} give the following example of a DGM. All factors have two endpoints, and we draw them as directed edges instead of the usual squares. We draw the edges in $E_{22}$ with dotted lines.
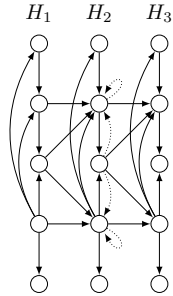
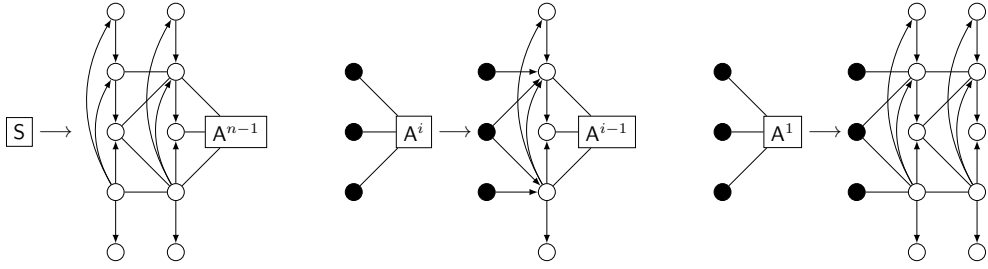
\shrinkdisplay{\begin{center}
\begin{tikzpicture}
\foreach \j in  {1,2,3} {
  \node at (\j,-0.5) {$H_\j$};
  \foreach \i in  {1,...,5} {
    \node[var](n\i\j) at (\j, -\i) {};
    }
}
\begin{scope}[every edge/.append style={-latex},every loop/.append style={-latex}]
\foreach \j in {1,2,3} {
\draw (n1\j) edge (n2\j);
\draw (n2\j) edge (n3\j);
\draw (n4\j) edge (n3\j);
\draw (n4\j) edge (n5\j);
\draw (n4\j) edge[bend left] (n1\j) edge[bend left](n2\j);
}
\draw (n21) edge (n22);
\draw (n31) edge (n22) edge (n42);
\draw (n41) edge (n42);
\draw (n22) edge (n23);
\draw (n32) edge (n23) edge (n43);
\draw (n42) edge (n43);
\begin{scope}[every edge/.append style={densely dotted}]
\draw (n22) edge[out=30,in=60,loop] (n22);
\draw (n32) edge[bend right] (n22) edge[bend left] (n42);
\draw (n42) edge[out=-30,in=-60,loop] (n42);
\end{scope}
\end{scope}
\end{tikzpicture}
\end{center}}

The resulting FGG:
\shrinkdisplay{\begin{align*}
\begin{tikzpicture}
\node[fac] {$\nt{S}$};
\end{tikzpicture}
&\longrightarrow
\begin{tikzpicture}
\begin{scope}[yshift=3cm]
\foreach \i in  {1,...,5} {
  \foreach \j in {1,2} {
    \node[var](n\i\j) at (\j, -\i) {};
  }
}
\begin{scope}[every edge/.append style={-latex},every loop/.append style={-latex}]
\foreach \j in {1,2} {
\draw (n1\j) edge (n2\j);
\draw (n2\j) edge (n3\j);
\draw (n4\j) edge (n3\j);
\draw (n4\j) edge (n5\j);
\draw (n4\j) edge[bend left] (n1\j) edge[bend left](n2\j);
}
\end{scope}
\draw (n21) edge (n22);
\draw (n31) edge (n22) edge (n42);
\draw (n41) edge (n42);
\node[fac] at (3,-3) {$\nt{A}^{n-1}$} edge (n22) edge (n32) edge (n42);
\end{scope}
\end{tikzpicture}
&\qquad
\begin{tikzpicture}
\begin{scope}[yshift=3cm]
\foreach \i in {2,3,4} {
    \node[ext](x\i) at (-1, -\i) {};
}
\end{scope}
\node[fac] {$\nt{A}^i$} edge (x2) edge (x3) edge (x4);
\end{tikzpicture}
&\longrightarrow
\begin{tikzpicture}
\begin{scope}[yshift=3cm]
\foreach \i in {2,3,4} {
    \node[ext](x\i) at (-1, -\i) {};
}
\foreach \i in {1,...,5} {
    \node[var](n\i) at (0, -\i) {};
}
\begin{scope}[every edge/.append style={-latex},every loop/.append style={-latex}]
\draw (x2) edge (n2);
\draw (x3) edge (n2) edge (n4);
\draw (x4) edge (n4);
\draw (n1) edge (n2);
\draw (n2) edge (n3);
\draw (n4) edge (n3);
\draw (n4) edge (n5);
\draw (n4) edge[bend left] (n1) edge[bend left](n2);
\end{scope}
\node[fac] at (1,-3) {$\nt{A}^{i-1}$} edge (n2) edge (n3) edge (n4);
\end{scope}
\end{tikzpicture}
%&&1 < i < n
%\\
&\qquad
\begin{tikzpicture}
\begin{scope}[yshift=3cm]
\foreach \i in {2,3,4} {
    \node[ext](x\i) at (-1, -\i) {};
}
\end{scope}
\node[fac] {$\nt{A}^1$} edge (x2) edge (x3) edge (x4);
\end{tikzpicture}
&\longrightarrow
\begin{tikzpicture}
\begin{scope}[yshift=3cm]
\foreach \i in {2,3,4} {
    \node[ext](x\i) at (0, -\i) {};
}
\foreach \i in  {1,...,5} {
  \foreach \j in {1,2} {
    \node[var](n\i\j) at (\j, -\i) {};
  }
}
\begin{scope}[every edge/.append style={-latex},every loop/.append style={-latex}]
\foreach \j in {1,2} {
\draw (n1\j) edge (n2\j);
\draw (n2\j) edge (n3\j);
\draw (n4\j) edge (n3\j);
\draw (n4\j) edge (n5\j);
\draw (n4\j) edge[bend left] (n1\j) edge[bend left](n2\j);
}
\end{scope}
\draw (x2) edge (n21);
\draw (x3) edge (n21) edge (n41);
\draw (x4) edge (n41);
\draw (n21) edge (n22);
\draw (n31) edge (n22) edge (n42);
\draw (n41) edge (n42);
\end{scope}
\end{tikzpicture}
\end{align*}}
where, in the middle rule, $1 < i < n$.
\end{example}

Running the algorithm of Theorem~\ref{thm:inference_cfg} would \emph{not} be guaranteed to achieve the same time complexity as that of \citep{bilmes+bartels:2003}, which searches through alternative ways of dividing the unrolled factor graph into time slices.

\subsection{Case--factor diagrams and sum--product networks}
\label{sec:cfd}

\newcommand{\scope}{\mathit{scope}}

Case--factor diagrams \citep{mcallester+:2008} and sum--product networks \citep{poon+domingos:2011} are compact representations of probability distributions over assignments to Boolean variables. They generalize both Markov networks and PCFGs.

Both formalisms represent models as rooted directed acyclic graphs (DAGs), with edges directed away from the root, in which some nodes mention variables. If $D$ is a DAG, for any node $v \in D$, let $\scope(v)$ be the set of variables mentioned in $v$ or any descendant of $v$.

\begin{definition}
A \emph{case--factor diagram (CFD) model} is a pair $(D, \Psi)$, where $D$ is a rooted DAG with root $r$, each of whose nodes is one of the following:
\begin{compactitem}
\item $\mathsf{case}(x)$ with two children $v_1$ and $v_2$, where $x$ is a variable not in $\scope(v_1) \cup \scope(v_2)$.
\item $\mathsf{factor}$ with two children $v_1$ and $v_2$, where $\scope(v_1) \cap \scope(v_2) = \emptyset$.
\item $\mathsf{unit}$ with no children.
\item $\mathsf{empty}$ with no children.
\end{compactitem}
And $\Psi : \scope(r) \rightarrow \mathbb{R}_{\geq 0}$ assigns a cost to each variable in $\scope(r)$.

A CFD model defines a probability distribution over assignments to its variables.
We compute quantities $q(v, \asst)$ and $Z(v)$ for each node $v$ as follows.
Let $v_1, v_2$ be the children of $v$, if any.
\begin{align*}
v &= \mathsf{case}(x) & q(v, \asst) &= \begin{cases}
e^{-\Psi(x)} \, q(v_1, \asst) & \text{if $\asst(x) = 1$}\\
q(v_2, \asst) & \text{if $\asst(x) = 0$}
\end{cases} & Z(v) &= e^{-\Psi(x)} \, Z(v_1) + Z(v_2) \\ % how can v_1 and v_2 have different scope?
v &= \mathsf{factor} & q(v, \asst) &= q(v_1, \asst) \, q(v_2, \asst) & Z(v) &= Z(v_1) \, Z(v_2) \\ % because v_1 and v_2 have no variables in common
v &= \mathsf{unit} & q(v, \asst) &= 1 & Z(v) &= 1 \\
v &= \mathsf{empty} & q(v, \asst) &= 0 & Z(v) &= 0 
\end{align*}
Define $q(\asst) = q(r, \asst)$ and $Z = Z(r)$. Then $P(\asst) = q(\asst) / Z$.
\end{definition}

\begin{proposition}
If $(D, \Psi)$ is a CFD model, there is a FGG $G$ such that $(D, \Psi)$ and $G$ have the same sum--product, and for any assignment $\asst$ of $(D, \Psi)$, there is a FGG $G_\asst$ such that the sum--product of $G \land G_\asst$ equals $q(\asst)$.
\end{proposition}

\begin{proof}
Given a CFD, we can construct a FGG where each node $v$ of the CFD becomes a different nonterminal symbol $\nt{D}_v$:
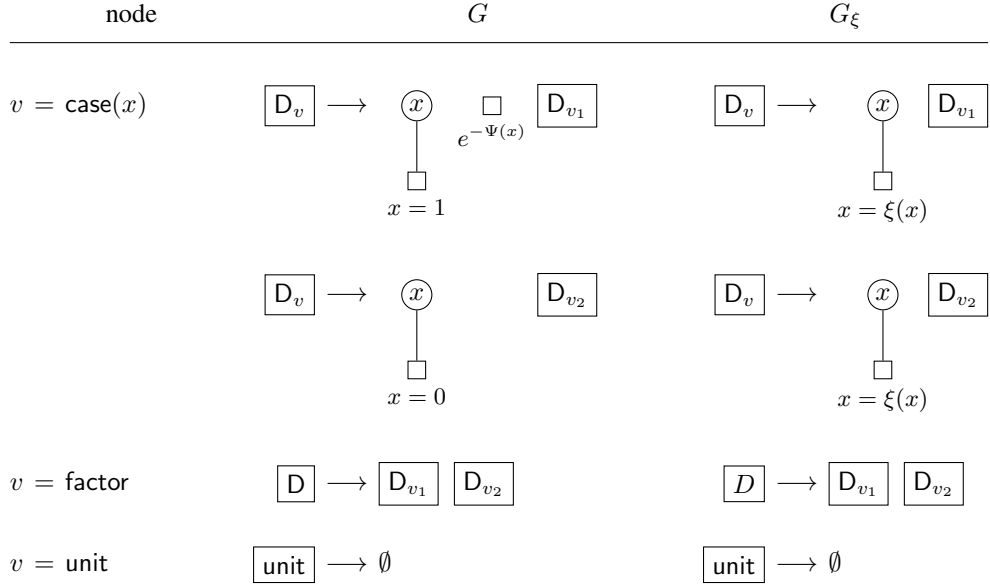
\[
\setlength{\arraycolsep}{2pt}
\begin{array}{@{}rcl@{\qquad\qquad}rcl@{\qquad\qquad}rcl@{}}
\multicolumn{3}{c}{\text{node}} & \multicolumn{3}{c}{G} & \multicolumn{3}{c}{G_\asst} \\[1ex]
\hline
\\[1ex]
v &=& \mathsf{case}(x) & 
\begin{tikzpicture}
\node[fac] {$\nt{D}_v$};
\end{tikzpicture}
&\longrightarrow&
\begin{tikzpicture}
\node[var](x) at (0,0) {$x$};
\node[fac,label=below:{$x=1$}] at (0,-1) {} edge (x);
\node[fac] at (2,0) {$\nt{D}_{v_1}$};
\node[fac,label=below:{$e^{-\Psi(x)}$}] at (1,0) {};
\end{tikzpicture}
&
\begin{tikzpicture}
\node[fac] {$\nt{D}_v$};
\end{tikzpicture}
&\longrightarrow&
\begin{tikzpicture}
\node[var](x) at (0,0) {$x$};
\node[fac,label=below:{$x=\asst(x)$}] at (0,-1) {} edge (x);
\node[fac] at (1,0) {$\nt{D}_{v_1}$};
\end{tikzpicture}
\\
\\[1ex]
&&
&\begin{tikzpicture}
\node[fac] {$\nt{D}_v$};
\end{tikzpicture}
&\longrightarrow&
\begin{tikzpicture}
\node[var](x) at (0,0) {$x$};
\node[fac,label=below:{$x=0$}] at (0,-1) {} edge (x);
\node[fac] at (2,0) {$\nt{D}_{v_2}$};
\end{tikzpicture}
&
\begin{tikzpicture}
\node[fac] {$\nt{D}_v$};
\end{tikzpicture}
&\longrightarrow&
\begin{tikzpicture}
\node[var](x) at (0,0) {$x$};
\node[fac,label=below:{$x=\asst(x)$}] at (0,-1) {} edge (x);
\node[fac] at (1,0) {$\nt{D}_{v_2}$};
\end{tikzpicture}
\\
\\[1ex]
v &=& \mathsf{factor}
&
\begin{tikzpicture}
\node[fac] {$\nt{D}$};
\end{tikzpicture}
&\longrightarrow&
\begin{tikzpicture}
\node[fac] at (0,0) {$\nt{D}_{v_1}$};
\node[fac] at (1,0) {$\nt{D}_{v_2}$};
\end{tikzpicture}
&
\begin{tikzpicture}
\node[fac] {$D$};
\end{tikzpicture}
&\longrightarrow&
\begin{tikzpicture}
\node[fac] at (0,0) {$\nt{D}_{v_1}$};
\node[fac] at (1,0) {$\nt{D}_{v_2}$};
\end{tikzpicture}
\\
\\[1ex]
v &=& \mathsf{unit} 
& 
\begin{tikzpicture}
\node[fac] {$\mathsf{unit}$};
\end{tikzpicture}
&\longrightarrow&
\emptyset
& 
\begin{tikzpicture}
\node[fac] {$\mathsf{unit}$};
\end{tikzpicture}
&\longrightarrow&
\emptyset
\end{array}\]
We do not create any rule with left-hand side $\mathsf{empty}$, so that any derivations that generate $\mathsf{empty}$ fail.
\end{proof}
The number of rules in $G$ is the number of nodes in $D$. Computing its sum--product is linear in the number of rules, just as computing the sum--product of $D$ is linear in the number of nodes.

\begin{definition}
A valid \emph{sum--product network} (SPN) is a rooted DAG whose nodes are each either:
\begin{compactitem}
\item $\mathsf{sum}(\lambda_1, \lambda_2)$ with two children $v_1$ and $v_2$, where $\scope(v_1) = \scope(v_2)$.
\item $\mathsf{product}$ with two children $v_1$ and $v_2$ such that no variable appears in one and negated in the other.
\item $x$ or $\bar{x}$ with no children.
\end{compactitem}
A valid SPN defines a distribution over assignments to its variables. For each node $v$, let $v_1, v_2$ be the children of $v$, if any.
\begin{align*}
v &= \mathsf{sum(\lambda_1, \lambda_2)} & q(v, \asst) &= \lambda_1 q(v_1, \asst) + \lambda_2 q(v_2, \asst) \\
v &= \mathsf{product} & q(v, \asst) &= q(v_1, \asst) \, q(v_2, \asst) \\
v &= x & q(v, \asst) &= \asst(x) \\
v &= \bar{x} & q(v, \asst) &= 1-\asst(x)
\end{align*}
\end{definition}

Converting a valid SPN to a FGG is straightforward, but the resulting FGG has a separate node for each occurrence of a variable $x$. The syntactic constraints in the definition of valid SPN ensure that in any graph with nonzero weight, all occurrences of $x$ have the same value.

\begin{proposition}
Any valid SPN $S$ can be converted into a FGG $G$ such that $S$ and $G$ have the same sum--product, and for any assignment $\asst$ of $S$, there is a FGG $G_\asst$ such that the sum--product of $G \land G_\asst$ equals $q(\asst)$. \end{proposition}
\begin{proof}
We construct a FGG where each node $v$ becomes a different nonterminal symbol $\nt{D}_v$:
\[
\setlength{\arraycolsep}{2pt}
\begin{array}{@{}rcl@{\qquad\qquad}rcl@{\qquad\qquad}rcl@{}}
\multicolumn{3}{c}{\text{node}} & \multicolumn{3}{c}{G} & \multicolumn{3}{c}{G_\asst} \\[1ex]
\hline
\\[1ex]
v &=& x & 
\begin{tikzpicture}
\node[fac] at (0,0) {$\nt{D}_v$}; % edge (x);
\end{tikzpicture}
&\longrightarrow&
\begin{tikzpicture}
\node[var] (x) {$x$};
\node[fac,label=below:{$x=1$}] at (1,0) {} edge (x);
\end{tikzpicture}
&
\begin{tikzpicture}
\node[fac] at (0,0) {$\nt{D}_v$}; % edge (x);
\end{tikzpicture}
&\longrightarrow&
\begin{tikzpicture}
\node[var] (x) {$x$};
\node[fac,label=below:{$x=\asst(x)$}] at (1,0) {} edge (x);
\end{tikzpicture}
\\
\\[1ex]
v &=& \bar{x} &
\begin{tikzpicture}
\node[fac] at (0,0) {$\nt{D}_v$}; % edge (x);
\end{tikzpicture}
&\longrightarrow&
\begin{tikzpicture}
\node[var] (x) {$x$};
\node[fac,label=below:{$x=0$}] at (1,0) {} edge (x);
\end{tikzpicture}
&
\begin{tikzpicture}
\node[fac] at (0,0) {$\nt{D}_v$}; % edge (x);
\end{tikzpicture}
&\longrightarrow&
\begin{tikzpicture}
\node[var] (x) {$x$};
\node[fac,label=below:{$x=\asst(x)$}] at (1,0) {} edge (x);
\end{tikzpicture}
\\
\\[1ex]
v &=& \mathsf{sum}(\lambda_1, \lambda_2) &
\begin{tikzpicture}
\node[fac] {$\nt{D}_v$};
\end{tikzpicture}
&\longrightarrow&
\begin{tikzpicture}
\node[fac,label=below:{$\lambda_1$}] {};
\node[fac] at (1,0) {$\nt{D}_{v_1}$};
\end{tikzpicture}
&
\begin{tikzpicture}
\node[fac] {$\nt{D}_v$};
\end{tikzpicture}
&\longrightarrow&
\begin{tikzpicture}
\node[fac] at (1,0) {$\nt{D}_{v_1}$};
\end{tikzpicture}
\\
&&&
\begin{tikzpicture}
\node[fac] {$\nt{D}_v$};
\end{tikzpicture}
&\longrightarrow&
\begin{tikzpicture}
\node[fac,label=below:{$\lambda_2$}] {};
\node[fac] at (1,0) {$\nt{D}_{v_2}$};
\end{tikzpicture}
&\begin{tikzpicture}
\node[fac] {$\nt{D}_v$};
\end{tikzpicture}
&\longrightarrow&
\begin{tikzpicture}
\node[fac] at (1,0) {$\nt{D}_{v_2}$};
\end{tikzpicture}
\\
\\[1ex]
v &=& \mathsf{product} &
\begin{tikzpicture}
\node[fac] {$\nt{D}_v$};
\end{tikzpicture}
&\longrightarrow&
\begin{tikzpicture}
\node[fac] at (0,0) {$\nt{D}_{v_1}$};
\node[fac] at (1,0) {$\nt{D}_{v_2}$};
\end{tikzpicture}
&
\begin{tikzpicture}
\node[fac] {$\nt{D}_v$};
\end{tikzpicture}
&\longrightarrow&
\begin{tikzpicture}
\node[fac] at (0,0) {$\nt{D}_{v_1}$};
\node[fac] at (1,0) {$\nt{D}_{v_2}$};
\end{tikzpicture}
\end{array} \]
\end{proof}
The number of rules in $G$ is the number of nodes in $S$. Computing its sum--product is linear in the number of rules, just as computing the sum--product of $S$ is linear in the number of nodes.

Further variations of SPNs have been proposed, in particular to generate repeated substructures \citep{stuhlmuller+goodman:2012,melibari+:2016}. Factored SPNs \citep{stuhlmuller+goodman:2012} are especially closely related to FGGs, in that they allow one part of a SPN to ``reference'' another, which is analogous to a nonterminal-labeled edge in a FGG.

CFDs and SPNs present a rather different, lower-level view of a model than the other formalisms surveyed here do. Whereas factor graphs and the other formalisms represent the model's \emph{variables} and the \emph{dependencies} among them, CFDs and SPNs (including factored SPNs) represent the \emph{computation} of the sum-product. For instance, converting a factor graph $H$ to a CFD or SPN requires forming a tree decomposition of $H$ \citep{mcallester+:2008}, and the resulting CFD/SPN's structure is that of the tree decomposition, not of $H$.

FGGs, in a sense, combine both points of view. Their derived graphs represent a model's variables and dependencies, while their derivation trees represent the computation of the sum-product. Thus, a factor graph $H$ can be trivially converted into a FGG $S \rightarrow H$, and, as can be seen in the translations given above, a CFD or SPN can also be converted to a FGG while preserving its structure.

\section{Proof of Proposition~\ref{lem:factorize}}
\label{app:factorize}

Let $H = (V, E)$ be a hypergraph. Recall that a \emph{tree decomposition} of $H$ is a tree whose nodes are called \emph{bags}, to each of which is associated a set of nodes, $V_B \subseteq V$, and (nonstandardly) a set of edges, $E_B \subseteq E$. The bags must satisfy the properties:
\begin{compactitem}
\item Node cover: $\bigcup_B V_B = V$.
\item Edge cover: for every edge $e \in E$, there is exactly one bag $B$ such that $e \in E_B$ and $\att(e) \subseteq V_B$.
\item Running intersection: if $v \in V_{B_1}$ and $v \in V_{B_2}$, then for every bag $B$ between $B_1$ and $B_2$, $v \in V_B$.
\end{compactitem}
The \emph{width} of a tree decomposition is $\max_B |V_B| - 1$, and the \emph{treewidth} of $H$ is the minimum width of any tree decomposition of $H$. 
A tree decomposition can always be made to have at most $n$ nodes without changing its width \citep{bodlaender:1996}.

\citet{chiang+:acl2013} give a parsing algorithm for HRGs that matches right-hand sides incrementally using their tree decompositions. They observe that this is related to the concept of binarization of context-free grammars. Here, we make this connection explicit by showing how to factorize a HRG.

For every rule $(X \rightarrow R)$, where $\bar{R}$ has $n_R$ nodes and treewidth at most $k$, form a tree decomposition of $\bar{R}$ with $n_R - k \leq n_R$ bags. Let the root of the tree decomposition be the bag containing all the external nodes of $R$. For each bag $B$, construct a rule $X_B \rightarrow R_B$ as follows.
\begin{compactitem}
\item If $B$ is the root bag, $X_B = X$; otherwise, $X_B$ is a fresh nonterminal symbol.
\item Add all nodes in $V_B$ and edges in $E_B$ to $R_B$.
\item If $B$ is the root bag, $R_B$'s external nodes are the same as $R$'s; if $B$ has parent $P$, let $R_B$'s external nodes be $V_P \cap V_B$.
\item For each child bag $B_i$, add a hyperedge with label $X_{B_i}$ and endpoints $V_B \cap V_{B_i}$. 
\end{compactitem}
This new FGG generates the same language as $G$. The number of rules is at most $\sum_{(X \rightarrow R) \in G} n_R = n_G$. Every right-hand side has at most $(k+1)$ nodes. %Finally, each rule has $k$ external nodes, so each nonterminal symbol has arity at most $k$.

\section{Supplement to Theorem~\ref{thm:inference_graph}}
\label{app:finitegrammar}

\subsection{An example}

\begin{example}
\label{eg:finitegrammar}
We show how to construct the factor graph corresponding to the following simple, nonreentrant FGG:
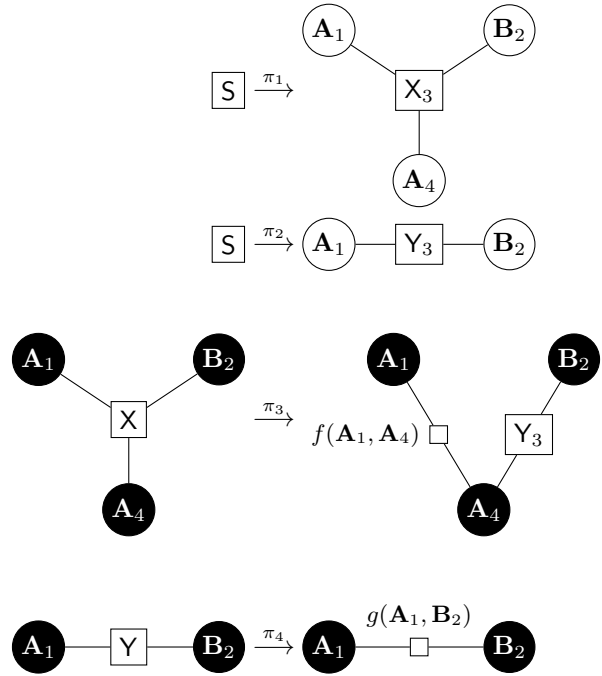
\begin{center}
$\begin{aligned}
    \begin{tikzpicture}[x=1.2cm,y=0.8cm] 
        \node[fac] at (0,0) { $\nt{S}$ }; 
    \end{tikzpicture} 
    &\xlongrightarrow{\pi_1}
    \begin{tikzpicture}[x=1.2cm,y=0.8cm] 
        \node[var] (a1) at (-1,1) {$\rv{A}_1$};
        \node[var] (b2) at (1,1) {$\rv{B}_2$};
        \node[var] (a4) at (0,-1.5) {$\rv{A}_4$};
        \node[fac] (x) at (0,0) {$\nt{X}_3$} edge (a1) edge (b2) edge (a4);
    \end{tikzpicture} 
\\
    \begin{tikzpicture}[x=1.2cm,y=0.8cm] 
        \node[fac] at (0,0) { $\nt{S}$ }; 
    \end{tikzpicture} 
    &\xlongrightarrow{\pi_2}
    \begin{tikzpicture}[x=1.2cm,y=0.8cm] 
        \node[var] (a1) at (-1,0) {$\rv{A}_1$};
        \node[var] (b2) at (1,0) {$\rv{B}_2$};
        \node[fac] (y) at (0,0) {$\nt{Y}_3$} edge (a1) edge (b2);
    \end{tikzpicture} \\
    \\[1ex]
    \begin{tikzpicture}[x=1.2cm,y=0.8cm] 
        \node[ext] (a1) at (-1,1) {$\rv{A}_1$};
        \node[ext] (b2) at (1,1) {$\rv{B}_2$};
        \node[ext] (a4) at (0,-1.5) {$\rv{A}_4$};
        \node[fac] (x) at (0,0) {$\nt{X}$} edge (a1) edge (b2) edge (a4);
    \end{tikzpicture}
    &\xlongrightarrow{\pi_3}
    \begin{tikzpicture}[x=1.2cm,y=0.8cm] 
        \node[ext] (a1) at (-1,1) {$\rv{A}_1$};
        \node[ext] (b2) at (1,1) {$\rv{B}_2$};
        \node[ext] (a4) at (0,-1.5) {$\rv{A}_4$};
        \node[fac,label=left:{$f(\rv{A}_1,\rv{A}_4)$}] at (-0.5,-0.25) {} edge (a1) edge (a4);
        \node[fac] (y) at (0.5,-0.25) {$\nt{Y}_3$} edge (a4) edge (b2);
    \end{tikzpicture} \\
    \\[1ex]
    \begin{tikzpicture}[x=1.2cm,y=0.8cm] 
        \node[ext] (a) at (-1,0) {$\rv{A}_1$};
        \node[fac] (y) at (0,0) {$\nt{Y}$} edge (a);
        \node[ext] (b) at (1,0) {$\rv{B}_2$} edge (y);
    \end{tikzpicture}
    &\xlongrightarrow{\pi_4} 
    \begin{tikzpicture}[x=1.2cm,y=0.8cm] 
        \node[ext] (a) at (-1,0) {$\rv{A}_1$};
        \node[ext] (b) at (1,0) {$\rv{B}_2$};
        \node[fac,label=above:{$g(\rv{A}_1,\rv{B}_2)$}] at (0,0) {} edge (a) edge (b);
    \end{tikzpicture}
\end{aligned}$
\end{center}

This grammar generates just two graphs:
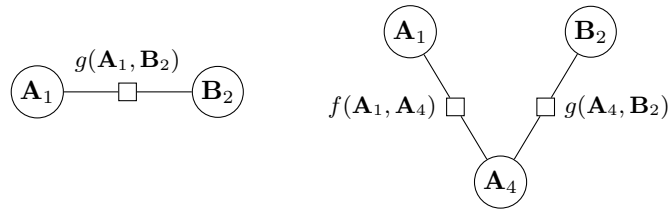
\begin{center}
$\begin{aligned}
    \begin{tikzpicture}[x=1.2cm,y=0.8cm]
        \node[var] (a) at (0,0) {$\rv{A}_1$};
        \node[var] (b) at (2,0) {$\rv{B}_2$};
        \node[fac,label=above:{$g(\rv{A}_1,\rv{B}_2)$}] at (1,0) {} edge (a) edge (b);
    \end{tikzpicture} \hspace{1cm}
    &
    \begin{tikzpicture}[x=1.2cm,y=0.8cm] 
        \node[var] (a1) at (0,1) {$\rv{A}_1$};
        \node[var] (a4) at (1,-1.5) {$\rv{A}_4$};
        \node[var] (b2) at (2,1) {$\rv{B}_2$};
        \node[fac,label=left:{$f(\rv{A}_1,\rv{A}_4)$}] at (0.5,-0.25) {} edge (a1) edge (a4);
        \node[fac,label=right:{$g(\rv{A}_4,\rv{B}_2)$}] at (1.5,-0.25) {} edge (a4) edge (b2);
    \end{tikzpicture}
\end{aligned}$
\end{center}
\end{example}

Applying the construction from Theorem~\ref{thm:inference_graph} gives the factor graph shown in Figure~\ref{fig:finitegrammar}. 
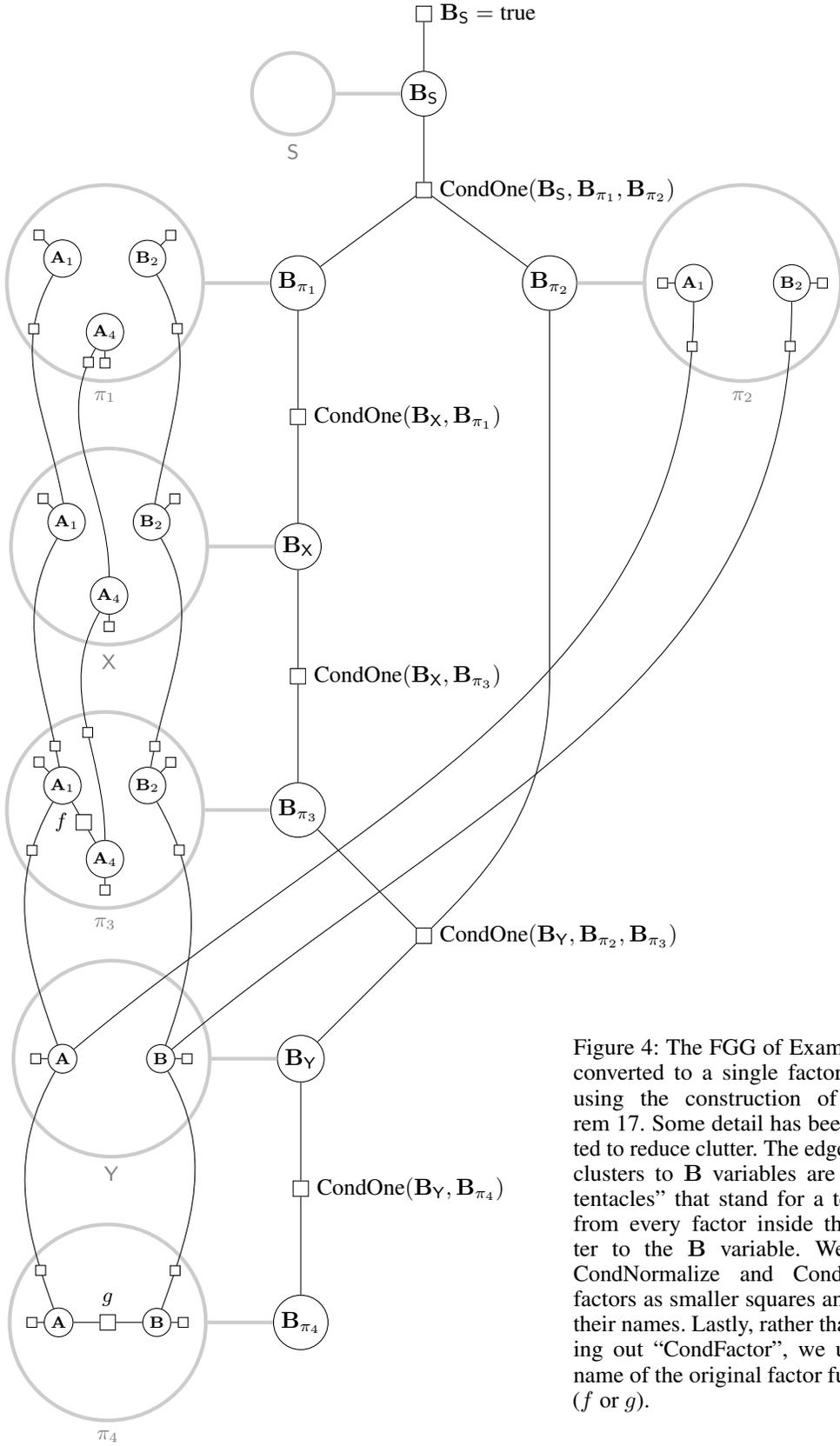
\begin{figure*}
\usetikzlibrary{calc,positioning,backgrounds}
\begin{center}
\begin{tikzpicture}
    % Binary indicator variables
    \tikzset{node distance=1.5cm}
    \begin{scope}[every label/.style={fill=white}]
    \node[fac,label=right:{$\rv{B}_{\nt{S}} = \text{true}$}] (root) {};
    \node[var] (b_s)   [below=0.75cm of root]       {$\rv{B}_{\nt{S}}$}
            edge (root);
    \node[fac,label=right:{$\text{CondOne}(\rv{B}_{\nt{S}},\rv{B}_{\pi_1},\rv{B}_{\pi_2})$}] (xor1)  [below=1cm of b_s] {}
            edge (b_s);
    \node[var] (b_pi1) [below left=1cm and 1.5cm of xor1]         {$\rv{B}_{\pi_1}$}
            edge (xor1);
    \node[var] (b_pi2) [below right=1cm and 1.5cm of xor1]        {$\rv{B}_{\pi_2}$}
            edge (xor1);
    \node[fac,label=right:{$\text{CondOne}(\rv{B}_{\nt{X}},\rv{B}_{\pi_1})$}] (xor2) [below=of b_pi1] {}
            edge (b_pi1);
    \node[var] (b_x)   [below=of xor2]              {$\rv{B}_{\nt{X}}$}
            edge (xor2);
    \node[fac,label=right:{$\text{CondOne}(\rv{B}_{\nt{X}},\rv{B}_{\pi_3})$}] (xor3) [below=of b_x] {}
            edge (b_x);
    \node[var] (b_pi3) [below=of xor3]              {$\rv{B}_{\pi_3}$}
            edge (xor3);
    \node[fac,label=right:{$\text{CondOne}(\rv{B}_{\nt{Y}},\rv{B}_{\pi_2},\rv{B}_{\pi_3})$}] (xor4)  [below right=1.5cm and 1.5cm of b_pi3] {}
            edge (b_pi3);
    \node[var] (b_y)   [below left=1.5cm and 1.5cm of xor4]    {$\rv{B}_{\nt{Y}}$}
            edge (xor4);
    \node[fac,label=right:{$\text{CondOne}(\rv{B}_{\nt{Y}},\rv{B}_{\pi_4})$}] (xor5)  [below=of b_y] {}
            edge (b_y);
    \node[var] (b_pi4) [below=of xor5]              {$\rv{B}_{\pi_4}$}
            edge (xor5);
    \end{scope}
    
    \coordinate (p) at ([yshift=-6cm]b_pi2);
    \draw (b_pi2) -- (p) .. controls +(0,-2) and +(1,1) .. (xor4);

    % Clusters
    \tikzset{node distance=1cm}
    \tikzset{cluster/.style={draw,ultra thick,black!20,circle,inner sep=0cm,minimum size=3cm}}
    \node[cluster] (c_s)    [minimum size=1.25cm,left=of b_s] [label={[black!50]below:$\nt{S}$}] {};
    \node[cluster] (c_pi1)  [left=of b_pi1] [label={[black!50]below:$\pi_1$}] {};
    \node[cluster] (c_pi2)  [right=of b_pi2] [label={[black!50]below:$\pi_2$}] {};
    \node[cluster] (c_x)    [left=of b_x] [label={[black!50]below:$\nt{X}$}] {};
    \node[cluster] (c_pi3)  [left=of b_pi3] [label={[black!50]below:$\pi_3$}] {};
    \node[cluster] (c_y)    [left=of b_y] [label={[black!50]below:$\nt{Y}$}] {};
    \node[cluster] (c_pi4)  [left=of b_pi4] [label={[black!50]below:$\pi_4$}] {};
    
    % Nodes inside clusters
    \tikzset{smallfac/.style={draw=black,solid,fill=white,rectangle,inner sep=0cm,minimum size=0.15cm}}
    \tikzset{every node/.style={font=\scriptsize}}
    % pi1
    \path (c_pi1.center) node[var] (pi1_a1) at +(150:0.75cm) {$\rv{A}_1$};
    \path (c_pi1.center) node[var] (pi1_b2) at +(30:0.75cm) {$\rv{B}_2$};
    \path (c_pi1.center) node[var] (pi1_a4) at +(270:0.75cm) {$\rv{A}_4$};
    \node[smallfac,above left=0.1cm of pi1_a1] {} edge (pi1_a1);
    \node[smallfac,above right=0.1cm of pi1_b2] {} edge (pi1_b2);
    \node[smallfac,below=0.1cm of pi1_a4] {} edge (pi1_a4);
    % pi2
    \path (c_pi2.center) node[var] (pi2_a1) at +(180:0.75cm) {$\rv{A}_1$};
    \path (c_pi2.center) node[var] (pi2_b2) at +(0:0.75cm) {$\rv{B}_2$};
    \node[smallfac,left=0.1cm of pi2_a1] {} edge (pi2_a1);
    \node[smallfac,right=0.1cm of pi2_b2] {} edge (pi2_b2);
    % X
    \path (c_x.center) node[var] (x_a1) at +(150:0.75cm) {$\rv{A}_1$};
    \path (c_x.center) node[var] (x_b2) at +(30:0.75cm) {$\rv{B}_2$};
    \path (c_x.center) node[var] (x_a4) at +(270:0.75cm) {$\rv{A}_4$};
    \node[smallfac,above left=0.1cm of x_a1] {} edge (x_a1);
    \node[smallfac,above right=0.1cm of x_b2] {} edge (x_b2);
    \node[smallfac,below=0.1cm of x_a4] {} edge (x_a4);
    % pi3
    \path (c_pi3.center) node[var] (pi3_a1) at +(150:0.75cm) {$\rv{A}_1$};
    \path (c_pi3.center) node[var] (pi3_b2) at +(30:0.75cm) {$\rv{B}_2$};
    \path (c_pi3.center) node[var] (pi3_a4) at +(270:0.75cm) {$\rv{A}_4$};
    \node[smallfac,above left=0.1cm of pi3_a1] {} edge (pi3_a1);
    \node[smallfac,above right=0.1cm of pi3_b2] {} edge (pi3_b2);
    \node[smallfac,below=0.1cm of pi3_a4] {} edge (pi3_a4);
    \node[fac,label=left:{$f$}] at ($(pi3_a1)!0.5!(pi3_a4)$) {}
            edge (pi3_a1)
            edge (pi3_a4);
    % Y
    \path (c_y.center) node[var] (y_a) at +(180:0.75cm) {$\rv{A}$};
    \path (c_y.center) node[var] (y_b) at +(0:0.75cm) {$\rv{B}$};
    \node[smallfac,left=0.1cm of y_a] {} edge (y_a);
    \node[smallfac,right=0.1cm of y_b] {} edge (y_b);
    % pi4
    \path (c_pi4.center) node[var] (pi4_a) at +(180:0.75cm) {$\rv{A}$};
    \path (c_pi4.center) node[var] (pi4_b) at +(0:0.75cm) {$\rv{B}$};
    \node[smallfac,left=0.1cm of pi4_a] {} edge (pi4_a);
    \node[smallfac,right=0.1cm of pi4_b] {} edge (pi4_b);
    \node[fac,label=above:{$g$}] at ($(pi4_a)!0.5!(pi4_b)$) {}
            edge (pi4_a)
            edge (pi4_b);
            
    % Equality constraints
    \begin{scope}[every edge/.style={draw=black}]
    \path (pi1_a1) edge[out=240,in=100] node[smallfac,pos=0.63,pos=0.24] {} (x_a1);
    \path (pi1_a4) edge[out=240,in=90]  node[smallfac,pos=0.36,pos=0.06] {} (x_a4);
    \path (pi1_b2) edge[out=300,in=80]  node[smallfac,pos=0.63,pos=0.24] {} (x_b2);
    \path (x_a1) edge[out=240,in=100] node[smallfac,pos=0.63,pos=0.92] {} (pi3_a1);
    \path (x_a4) edge[out=240,in=90]  node[smallfac,pos=0.36,pos=0.55] {} (pi3_a4);
    \path (x_b2) edge[out=300,in=80]  node[smallfac,pos=0.63,pos=0.92] {} (pi3_b2);
    \path (pi3_a1) edge[out=240,in=110] node[smallfac,pos=0.55,pos=0.2] {} (y_a);
    \path (pi3_b2) edge[out=300,in=70]  node[smallfac,pos=0.55,pos=0.2] {} (y_b);
    \path (y_a) edge[out=240,in=110] node[smallfac,pos=0.55,pos=0.85] {} (pi4_a);
    \path (y_b) edge[out=300,in=70]  node[smallfac,pos=0.55,pos=0.85] {} (pi4_b);
    \path (pi2_a1) edge[out=270,in=30,in=40] node[smallfac,pos=0.19,pos=0.04] {} (y_a);
    \path (pi2_b2) edge[out=280,in=45,out=270,in=40] node[smallfac,pos=0.24,pos=0.04] {} (y_b);
    \end{scope}
    
    % meta-tentacles from clusters to B variables
    \begin{scope}[every edge/.style={draw,solid,ultra thick,black!20}]
    \draw (c_s) edge (b_s);
    \draw (c_x) edge (b_x);
    \draw (c_y) edge (b_y);
    \draw (c_pi1) edge (b_pi1);
    \draw (c_pi2) edge (b_pi2);
    \draw (c_pi3) edge (b_pi3);
    \draw (c_pi4) edge (b_pi4);
    \end{scope}

\node at (1.9in,-7.25in) {\parbox{2in}{\caption{The FGG of Example~\ref{eg:finitegrammar}, converted to a single factor graph using the construction of Theorem~\ref{thm:inference_graph}. Some detail has been omitted to reduce clutter. 
The edges from clusters to $\rv{B}$ variables are ``meta-tentacles'' that stand for a tentacle from every factor inside the cluster to the $\rv{B}$ variable. We draw $\text{CondNormalize}$ and $\text{CondEquals}$ factors as smaller squares and omit their names. 
Lastly, rather than writing out ``$\text{CondFactor}$'', we use the name of the original factor function ($f$ or $g$).}
\label{fig:finitegrammar}
}};

\end{tikzpicture}
\end{center}
\end{figure*}

\subsection{Complexity of inference}

As noted in Section~\ref{sec:inference_graph}, the purpose of this conversion to a single factor graph is to make inference possible with infinite variable domains; after converting to a factor graph, existing, possibly approximate, inference methods can be applied. But with finite variable domains, an algorithm like variable elimination would not be appropriate because this conversion has the potential to increase treewidth dramatically.

In the proof of Theorem~\ref{thm:inference_cfg}, we constructed the \emph{nonterminal graph}, which has a node for every nonterminal and an edge from $X$ to $Y$ iff there is a rule $X\rightarrow R$ where $R$ has an edge labeled $Y$. For a nonreentrant FGG, the nonterminal graph is always a DAG. If, for each $X \in {N \setminus S}$, $X$ appears in the right-hand side of exactly one rule, then the nonterminal graph is a tree.

When the nonterminal graph is a tree, we can construct a tree decomposition by making one bag for each cluster, and one bag for each $\text{CondOne}$ factor. The bag for a cluster $C$ contains all the variables in $C$, along with $\rv{B}_C$ and all the $\text{CondNormalize}$ and $\text{CondFactor}$ edges associated with $C$. The bag for a $\text{CondOne}$ factor will contain all the $\rv{B}$ variables used by that $\text{CondOne}$ factor, all the $\text{CondEquals}$ edges connecting clusters involved in that $\text{CondOne}$ factor, and all the variables connected to those $\text{CondEquals}$ edges.

Exact inference on this tree decomposition is very similar to the algorithm described in Theorem~\ref{thm:inference_cfg}. However, a na\"{i}ve application of variable elimination will still be less efficient than that algorithm, since the $\text{CondOne}$ factors connect $|P^X|+1$ binary variables, requiring a loop over $2^{|P^X|+1}$ assignments. All but $|P^X|+1$ of these assignments have zero weight, so in fact we can process these factors much faster; modifying the variable elimination algorithm to account for this and the $\text{CondEquals}$ constraints would give us something almost identical to the algorithm of Theorem~\ref{thm:inference_cfg}.

In the DAG case, this simple tree decomposition is not possible. The factor graph $H$ has the nonterminal graph as a minor, so the treewidth of the nonterminal graph is a lower bound on the treewidth of $H$ \citep[Lemma 16]{bodlaender:1998}. In the worst case, this could be $|N|$. 

\subsection{Detailed proof of correctness}

If $G$ is a FGG and $H$ is the factor graph that results from the construction of Theorem~\ref{thm:inference_graph}, we can show that they have the same sum--product $Z_G = Z_H$.

The sum--product $Z_H$ can be computed in the usual way, by summing over all assignments to the variables and, for each assignment, taking the product over all of the factors:
\begin{align*}
    Z_H &= \sum_{\asst \in \Asst_H} \prod_{e \in H} F(e)(\asst(e)).
\end{align*}

The summation over assignments $\asst$ includes many possible settings of the $\rv{B}$ variables. But the $\text{CondOne}$ factors tell us that, if the assignment to the $\rv{B}$ variables does not give us a valid derivation, then the weight of that assignment will be 0. Therefore, we only need to sum over assignments to the $\rv{B}$ variables which represent a valid derivation, and so we can express the sum--product using a sum over derivations rather than a sum over assignments to $\rv{B}$ variables. Let $\asst_{\rv{B}}$ represent the assignment to the $\rv{B}$ variables. Then:

\begin{align*}
    Z_H &= \sum_{\deriv \in \mathcal{D}(G)} \sum_{\substack{\asst \in \Asst_H \\ \text{$\asst_{\rv{B}}$ consistent with $\deriv$}}} \prod_{e \in H} F(e)(\asst(e)).
\end{align*}
(Note that the product over $e \in H$ can ignore all $\text{CondOne}$ factors, since when the assignment to the $\rv{B}$ variables is consistent with some derivation, they all have value 1.)

We can associate a derivation $\deriv$ with the subset of clusters in $H$ corresponding to the nonterminals and rules which were used in the derivation; call this $\mathcal{C}_\deriv$. For any $\deriv$, all the variables in $H$ are divided into three parts: those that belong to clusters in $\mathcal{C}_\deriv$ (call this $V_\deriv$), those that belong to clusters not in $\mathcal{C}_\deriv$ (call this $V_{\overline \deriv}$), and the $\rv{B}$ variables (which don't belong to any cluster). 
Let $\asst_{\rv{B},D}$ be the unique assignment to the $\rv{B}$ variables that is consistent with $\deriv$.
Let $\Asst_\deriv$ be the set of all assignments extending $\asst_{\rv{B},D}$ with assignments to $V_\deriv$, and let $\Asst_{\overline\deriv}$ be the set of all assignments extending $\asst_{\rv{B},D}$ with assignments to~$V_{\overline\deriv}$. 

Let $E_\deriv$ be the set of factors involving a variable in $V_\deriv$, and let $E_{\overline\deriv}$ be the set of factors involving a variable in $V_{\overline\deriv}$. 
Because any factors between $V_\deriv$ and $V_{\overline\deriv}$ are $\text{CondEquals}$ factors with value 1 (since their $\rv{B}$ variable is false), we can ignore them. Similarly, the only factors which don't involve either $V_\deriv$ or $V_{\overline\deriv}$ are the $\text{CondOne}$ factors, which we are already ignoring. This allows us to rewrite the sum--product as
\begin{align*}
    Z_H &= \sum_{\deriv \in \mathcal{D}} \Biggl(\underbrace{\sum_{\asst \in \Asst_\deriv} \prod_{e \in E_\deriv} F(e)(\asst(e))}_{Z_\deriv}\Biggr) \Biggl(\underbrace{\sum_{\asst \in \Asst_{\overline\deriv}} \prod_{e \in E_{\overline\deriv}} F(e)(\asst(e))}_{Z_{\overline\deriv}}\Biggr).
\end{align*}

Consider $Z_{\overline\deriv}$ first. All $\text{CondFactor}$ and $\text{CondEquals}$ factors in $E_{\overline\deriv}$ have value 1 and can be ignored, leaving only $\text{CondNormalize}$ factors. Because these place a probability distribution $p_v$ on each variable $v$ in an unused cluster, those variables all sum out:
\begin{align*}
Z_{\overline\deriv} &= \sum_{\asst \in \Asst_{\overline\deriv}} \prod_{C_X \not\in \mathcal{C}_\deriv} \prod_{v \in C_X} \text{CondNormalize}_v(\rv{B}_X, v) \prod_{C_\pi \not\in \mathcal{C}_\deriv} \prod_{v \in C_\pi} \text{CondNormalize}_v(\rv{B}_\pi, v) \\
&= \sum_{\asst \in \Asst_{\overline\deriv}} \prod_{C_X \not\in \mathcal{C}_\deriv} \prod_{v \in C_X} p_v(\xi(v)) \prod_{C_\pi \not\in \mathcal{C}_\deriv} \prod_{v \in C_\pi} p_v(\xi(v)) \\
&= \prod_{C_X \not\in \mathcal{C}_\deriv} \prod_{v \in C_X} \left(\sum_{x \in\domain(v)} p_v(x)\right) \prod_{C_\pi \not\in \mathcal{C}_\deriv} \prod_{v \in C_\pi} \left(\sum_{x \in\domain(v)} p_v(x)\right)  \\
&= 1.
\end{align*}

Now consider $Z_\deriv$. All $\text{CondNormalize}$ factors in $E_\deriv$ have value 1 and can be ignored, leaving only $\text{CondEquals}$ and $\text{CondFactor}$ factors. Let $H_\deriv$ be the derived graph of $\deriv$. We can think of the derivation as merging pairs of nodes in $V_\deriv$, so that a single node $v \in H_\deriv$ may correspond to several ``copies'' in $V_\deriv$. However, the $\text{CondEquals}$ constraints ensure that all copies of $v$ have the same value. Therefore, instead of summing over the assignments to $V_\deriv$, we can simply sum over the assignments to $H_\deriv$ (and omit $\text{CondEquals}$ factors):
\begin{align*}
    Z_\deriv &= \sum_{\asst \in \Asst_{H_\deriv}} \prod_{C_\pi\in\mathcal{C}_\deriv} \prod_{e\in\pi} \text{CondFactor}_e(\rv{B}_\pi, \asst(\att(e))) \\
    &= \sum_{\asst \in \Asst_{H_\deriv}} \prod_{C_\pi\in\mathcal{C}_\deriv} \prod_{e\in\pi} F(e)(\asst(\att(e))) \\
    &= \sum_{\asst \in \Asst_{H_\deriv}} \prod_{e \in H_\deriv} F(e)(\asst(\att(e))).
\end{align*}
So, finally, the sum--product of $H$ can be rewritten as:
\begin{align*}
    Z_H &= \sum_{\deriv \in \mathcal{D}(G)} \, \sum_{\asst \in \Asst_{H_\deriv}} \,
            \prod_{e \in H_\deriv} F(e)(\asst(att(e))) \\
        &= \sum_{\deriv \in \mathcal{D}(G)} \, \sum_{\asst \in \Asst_{H_\deriv}} w_G(\deriv, \asst) \\
        &= Z_G.
\end{align*}
\section{Proof of Theorem~\ref{thm:undecidable}}
\label{sec:tm}

\newcommand{\blank}{\textrm{\textvisiblespace}}

Let $\Gamma$ be a finite alphabet containing a blank symbol $(\blank)$, and let $k=|\Gamma|$. Number the symbols in $\Gamma$ as $\gamma_0 = \blank, \gamma_1, \gamma_2, \ldots, \gamma_{k-1}$. Define an encoding for strings over $\Gamma$:
\begin{align*}
\langle \epsilon \rangle &= 0 \\
\langle \gamma_i w \rangle &= i + k\cdot\langle w \rangle.
\end{align*}
Note that strings that differ only in the number of trailing blanks have the same encoding.

\renewcommand{\div}{\ensuremath{\mathbin{/\!/}}}
\renewcommand{\mod}{\ensuremath{\mathbin{\%}}}

We write $x \div k$ for $\lfloor x/k \rfloor$ and $x \mod k = x - x \div k \cdot k$.

Let $M$ be a Turing machine with doubly-infinite tape, input alphabet $\Sigma$, tape alphabet $\Gamma$, start state $q_0$, transition function $\delta$, accept state $q_{\text{accept}}$, and reject state $q_{\text{reject}}$. For any input string $w \in \Sigma^\ast$, construct the following rules, where the $\rv{q}$ nodes track the Turing machine's state, the $\rv{u}$ nodes track the reverse of the tape to the left of the head, and the $\rv{v}$ nodes track the tape from the head rightward:

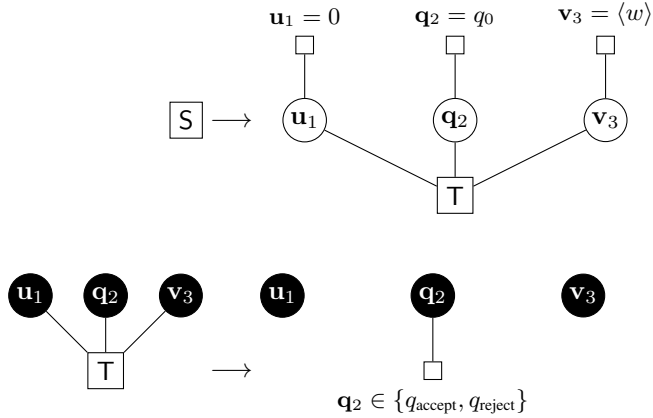
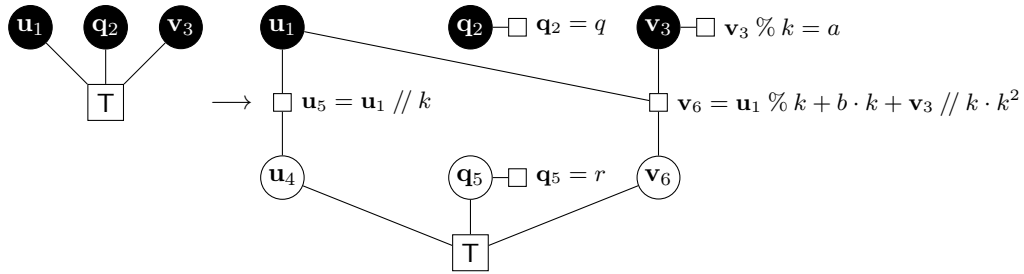
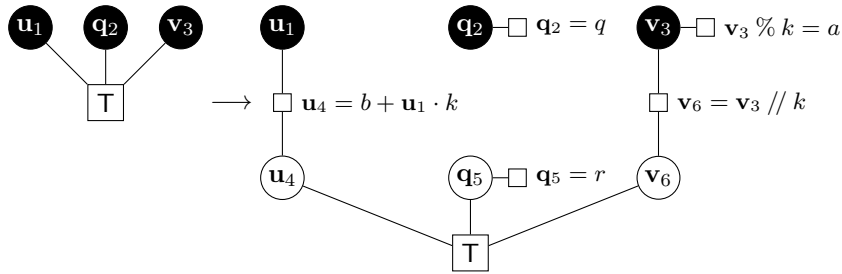
\begin{align*}
\begin{tikzpicture}
\node[fac] {$\nt{S}$};
\end{tikzpicture}
&\longrightarrow
\begin{tikzpicture}[x=2cm]
\node[var](u) at (-1,0) {$\rv{u}_1$};
\node[var](q) at (0,0) {$\rv{q}_2$};
\node[var](v) at (1,0) {$\rv{v}_3$};
\node[fac,label=above:{$\rv{u}_1 = 0$}] at (-1,1) {} edge (u);
\node[fac,label=above:{$\rv{q}_2 = q_0$}] at (0,1) {} edge (q);
\node[fac,label=above:{$\rv{v}_3 = \langle w\rangle$}] at (1,1) {} edge (v);
\node[fac] at (0,-1) {$\nt{T}$} edge (u) edge (q) edge (v);
\end{tikzpicture}
\\
\\[1ex]
\begin{tikzpicture}
\node[ext](u) at (-1,1) {$\rv{u}_1$};
\node[ext](q) at (0,1) {$\rv{q}_2$};
\node[ext](v) at (1,1) {$\rv{v}_3$};
\node[fac] at (0, 0) {$\nt{T}$} edge (u) edge (q) edge (v);
\end{tikzpicture}
&\longrightarrow
\begin{tikzpicture}[x=2cm]
\node[ext](u) at (-1,1) {$\rv{u}_1$};
\node[ext](q) at (0,1) {$\rv{q}_2$};
\node[ext](v) at (1,1) {$\rv{v}_3$};
\node[fac,label=below:{$\rv{q}_2 \in \{q_{\text{accept}}, q_{\text{reject}}\}$}] at (0,0) {} edge (q);
\end{tikzpicture}
\\
\intertext{For each transition $\delta(q, a) = (r, b, \text{L})$:}
\begin{tikzpicture}
\node[ext](u) at (-1,1) {$\rv{u}_1$};
\node[ext](q) at (0,1) {$\rv{q}_2$};
\node[ext](v) at (1,1) {$\rv{v}_3$};
\node[fac] at (0, 0) {$\nt{T}$} edge (u) edge (q) edge (v);
\end{tikzpicture}
&\longrightarrow
\begin{tikzpicture}[x=2.5cm,y=1cm]
\node[ext](u1) at (-1,1) {$\rv{u}_1$};
\node[ext](q1) at (0,1) {$\rv{q}_2$};
\node[ext](v1) at (1,1) {$\rv{v}_3$};
\node[fac,label=right:{$\rv{q}_2 = q$}] at (0.25,1) {} edge (q1);
\node[var](u2) at (-1,-1) {$\rv{u}_4$};
\node[var](q2) at (0,-1) {$\rv{q}_5$};
\node[fac,label=right:{$\rv{q}_5 = r$}] at (0.25,-1) {} edge (q2);
\node[var](v2) at (1,-1) {$\rv{v}_6$};
\node[fac,label=right:{$\rv{u}_5 = \rv{u}_1 \div k$}] at (-1,0) {} edge (u1) edge (u2);
\node[fac,label=right:{$\rv{v}_3 \mod k = a$}] at (1.25,1) {} edge (v1);
\node[fac,label=right:{$\rv{v}_6 = \rv{u}_1 \mod k + b \cdot k + \rv{v}_3 \div k \cdot k^2$}] at (1,0) {} edge (u1) edge (v1) edge (v2);
\node[fac] at (0,-2) {$\nt{T}$} edge (u2) edge (q2) edge (v2);
\end{tikzpicture}
\\
\intertext{For each transition $\delta(q, a) = (r, b, \text{R})$:}
\begin{tikzpicture}
\node[ext](u) at (-1,1) {$\rv{u}_1$};
\node[ext](q) at (0,1) {$\rv{q}_2$};
\node[ext](v) at (1,1) {$\rv{v}_3$};
\node[fac] at (0, 0) {$\nt{T}$} edge (u) edge (q) edge (v);
\end{tikzpicture}
&\longrightarrow
\begin{tikzpicture}[x=2.5cm,y=1cm]
\node[ext](u1) at (-1,1) {$\rv{u}_1$};
\node[ext](q1) at (0,1) {$\rv{q}_2$};
\node[fac,label=right:{$\rv{q}_2 = q$}] at (0.25,1) {} edge (q1);
\node[ext](v1) at (1,1) {$\rv{v}_3$};
\node[var](u2) at (-1,-1) {$\rv{u}_4$};
\node[var](q2) at (0,-1) {$\rv{q}_5$};
\node[fac,label=right:{$\rv{q}_5 = r$}] at (0.25,-1) {} edge (q2);
\node[var](v2) at (1,-1) {$\rv{v}_6$};
\node[fac,label=right:{$\rv{v}_3 \mod k = a$}] at (1.25,1) {} edge (v1);
\node[fac,label=right:{$\rv{u}_4 = b + \rv{u}_1 \cdot k$}] at (-1,0) {} edge (u1) edge (u2);
\node[fac,label=right:{$\rv{v}_6 = \rv{v}_3 \div k$}] at (1,0) {} edge (v1) edge (v2);
\node[fac] at (0,-2) {$\nt{T}$} edge (u2) edge (q2) edge (v2);
\end{tikzpicture}
\end{align*}

The sum-product of this FGG is 1 if $M$ halts on $w$, 0 otherwise. Therefore, computing the sum-product of an FGG is undecidable.

The operations $+, \cdot, \div, \mod$ and $=$ can be further reduced to just the successor relation and equality with zero, as shown below.

{\tikzset{x=1.2cm,y=0.6cm}

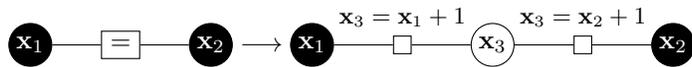
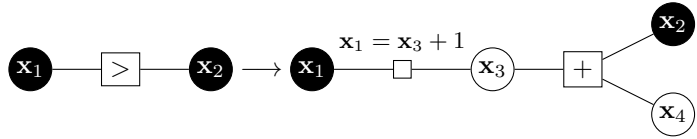
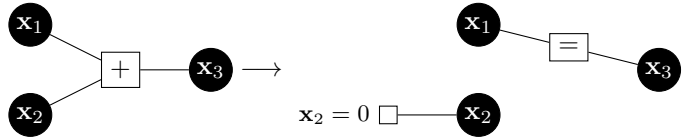
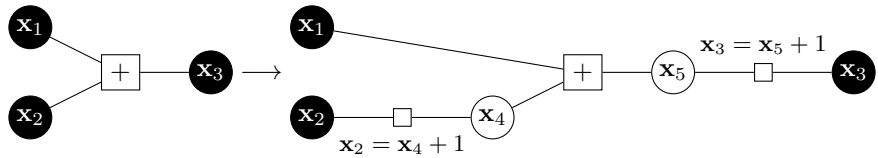
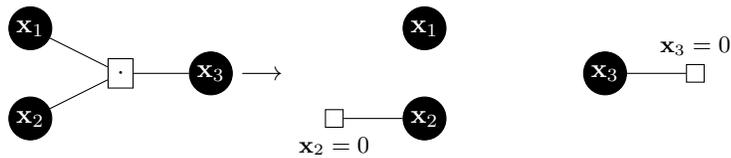
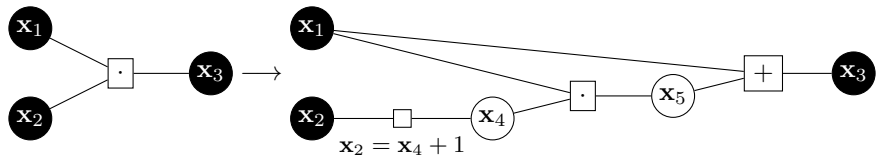
\begin{align*}
\begin{tikzpicture}
\node[ext](x) at (0,0) {$\rv{x}_1$};
\node[ext](y) at (2,0) {$\rv{x}_2$};
\node[fac] at (1,0) {$\nt{=}$} edge (x) edge (y);
\end{tikzpicture}
&
\longrightarrow
\begin{tikzpicture}
\node[ext](x) at (0,0) {$\rv{x}_1$};
\node[ext](y) at (4,0) {$\rv{x}_2$};
\node[var](z) at (2,0) {$\rv{x}_3$};
\node[fac,label=above:{$\rv{x}_3=\rv{x}_1+1$}] at (1, 0) {} edge (x) edge (z);
\node[fac,label=above:{$\rv{x}_3=\rv{x}_2+1$}] at (3, 0) {} edge (y) edge (z);
\end{tikzpicture}
\\
\\[1ex]
\begin{tikzpicture}
\node[ext](x) at (0,0) {$\rv{x}_1$};
\node[ext](y) at (2,0) {$\rv{x}_2$};
\node[fac] at (1,0) {$\nt{>}$} edge (x) edge (y);
\end{tikzpicture}
&
\longrightarrow
\begin{tikzpicture}
\node[ext](x) at (0,0) {$\rv{x}_1$};
\node[ext](y) at (4,1) {$\rv{x}_2$};
\node[var](x1) at (2,0) {$\rv{x}_3$};
\node[var](z) at (4,-1) {$\rv{x}_4$};
\node[fac,label=above:{$\rv{x}_1=\rv{x}_3+1$}] at (1, 0) {} edge (x) edge (x1);
\node[fac] at (3,0) {$+$} edge (x1) edge (z) edge (y);
\end{tikzpicture}
\\
\\[1ex]
\begin{tikzpicture}
\node[ext](z) at (1,0) {$\rv{x}_3$};
\node[ext](x) at (-1,1) {$\rv{x}_1$};
\node[ext](y) at (-1,-1) {$\rv{x}_2$};
\node[fac] at (0,0) {$\nt{+}$} edge (x) edge (y) edge (z);
\end{tikzpicture}
&
\longrightarrow
\begin{tikzpicture}
\node[ext](z) at (1,0) {$\rv{x}_3$};
\node[ext](x) at (-1,1) {$\rv{x}_1$};
\node[ext](y) at (-1,-1) {$\rv{x}_2$};
\node[fac,label=left:{$\rv{x}_2=0$}] at (-2, -1) {} edge (y);
\node[fac] at (0,0.5) {$=$} edge (x) edge (z);
\end{tikzpicture}
\\
\\[1ex]
\begin{tikzpicture}
\node[ext](z) at (1,0) {$\rv{x}_3$};
\node[ext](x) at (-1,1) {$\rv{x}_1$};
\node[ext](y) at (-1,-1) {$\rv{x}_2$};
\node[fac] at (0,0) {$\nt{+}$} edge (x) edge (y) edge (z);
\end{tikzpicture}
&
\longrightarrow
\begin{tikzpicture}
\node[ext](z) at (6,0) {$\rv{x}_3$};
\node[var](z1) at (4,0) {$\rv{x}_{5}$};
\node[ext](x) at (0,1) {$\rv{x}_1$};
\node[ext](y) at (0,-1) {$\rv{x}_2$};
\node[var](y1) at (2,-1) {$\rv{x}_{4}$};
\node[fac] at (3,0) {$\nt{+}$} edge (x) edge (y1) edge (z1);
\node[fac,label=above:{$\rv{x}_3 = \rv{x}_{5}+1$}] at (5,0) {} edge (z) edge (z1);
\node[fac,label=below:{$\rv{x}_2 = \rv{x}_{4}+1$}] at (1,-1) {} edge (y) edge (y1);
\end{tikzpicture}
\\
\\[1ex]
\begin{tikzpicture}
\node[ext](z) at (1,0) {$\rv{x}_3$};
\node[ext](x) at (-1,1) {$\rv{x}_1$};
\node[ext](y) at (-1,-1) {$\rv{x}_2$};
\node[fac] at (0,0) {$\cdot$} edge (x) edge (y) edge (z);
\end{tikzpicture}
&
\longrightarrow
\begin{tikzpicture}
\node[ext](z) at (1,0) {$\rv{x}_3$};
\node[ext](x) at (-1,1) {$\rv{x}_1$};
\node[ext](y) at (-1,-1) {$\rv{x}_2$};
\node[fac,label=below:{$\rv{x}_2=0$}] at (-2, -1) {} edge (y);
\node[fac,label=above:{$\rv{x}_3=0$}] at (2,0) {} edge (z);
\end{tikzpicture}
\\
\\[1ex]
\begin{tikzpicture}
\node[ext](z) at (1,0) {$\rv{x}_3$};
\node[ext](x) at (-1,1) {$\rv{x}_1$};
\node[ext](y) at (-1,-1) {$\rv{x}_2$};
\node[fac] at (0,0) {$\cdot$} edge (x) edge (y) edge (z);
\end{tikzpicture}
&
\longrightarrow
\begin{tikzpicture}
\node[ext](z) at (6,0) {$\rv{x}_3$};
\node[var](z1) at (4,-0.5) {$\rv{x}_{5}$};
\node[ext](x) at (0,1) {$\rv{x}_1$};
\node[ext](y) at (0,-1) {$\rv{x}_2$};
\node[var](y1) at (2,-1) {$\rv{x}_{4}$};
\node[fac] at (3,-0.5) {$\nt{\cdot}$} edge (x) edge (y1) edge (z1);
\node[fac] at (5,0) {$\nt{+}$} edge (x) edge (z) edge (z1);
\node[fac,label=below:{$\rv{x}_2 = \rv{x}_{4}+1$}] at (1,-1) {} edge (y) edge (y1);
\end{tikzpicture}
\\
\\[1ex]
\intertext{Integer division and remainder can both be computed using the rule:}
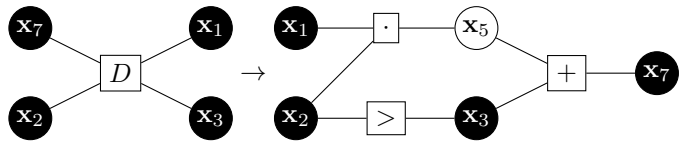
\begin{tikzpicture}
\node[ext](x) at (-1,1) {$\rv{x}_7$};
\node[ext](y) at (-1,-1) {$\rv{x}_2$};
\node[ext](q) at (1,1) {$\rv{x}_1$};
\node[ext](r) at (1,-1) {$\rv{x}_3$};
\node[fac] at (0,0) {$D$} edge (x) edge (y) edge (q) edge (r);
\end{tikzpicture}
&\rightarrow
\begin{tikzpicture}
\node[ext](x) at (2,1) {$\rv{x}_1$};
\node[ext](m) at (2,-1) {$\rv{x}_2$};
\node[var](xm) at (4,1) {$\rv{x}_5$};
\node[fac] at (3,1) {$\cdot$} edge (x) edge (m) edge (xm);
\node[ext](r) at (4,-1) {$\rv{x}_3$};
\node[ext](y) at (6,0) {$\rv{x}_7$};
\node[fac] at (3,-1) {$>$} edge (r) edge (m);
\node[fac] at (5,0) {$+$} edge (r) edge (xm) edge (y);
\end{tikzpicture}
\end{align*}
where $\rv{x}_7$ is the dividend, $\rv{x}_2$ is the divisor, $\rv{x}_1$ is the quotient, and $\rv{x}_3$ is the remainder.
}

\end{document}